\documentclass[sigconf]{acmart}

\usepackage[utf8]{inputenc}
\usepackage{amsmath}
\usepackage{tikz}
\usepackage{todonotes}
\usepackage[noend]{algpseudocode}
\usepackage[ruled]{algorithm}
\usepackage{thmtools}
\usepackage{thm-restate}
\usepackage{booktabs}
\usepackage{multirow}
\usepackage{comment}
\usepackage{subfig}
\usepackage{caption}

\usepackage{pifont}
\declaretheorem[name=Fact]{fact}

\renewcommand{\emptyset}{\varnothing}

\newif\iffull
\fulltrue

\newcommand{\R}{\mathbb{R}}
\newcommand{\N}{\mathbb{N}}
\newcommand{\feats}{\mathcal{X}}
\newcommand{\labels}{\mathcal{Y}}
\newcommand{\dataset}{\mathcal{D}}
\newcommand{\dtrain}{\dataset_{\textit{train}}}
\newcommand{\dtest}{\dataset_{\textit{test}}}
\DeclareMathOperator*{\argmin}{arg\,min}
\newcommand{\spread}{\psi_p}

\newcommand{\supp}{\textit{supp}}
\newcommand{\opt}{\textit{opt}}
\newcommand{\var}[1]{\textit{#1}}
\newcommand{\writtenbyLC}[1]{#1}
\newcommand{\revise}[1]{#1}

\title{Verifiable Learning for Robust Tree Ensembles}

\AtBeginDocument{%
 }

\copyrightyear{2023}
\acmYear{2023}
\setcopyright{rightsretained}
\acmConference[CCS '23]{Proceedings of the 2023 ACM SIGSAC Conference on Computer and Communications Security}{November 26--30, 2023}{Copenhagen, Denmark}
\acmBooktitle{Proceedings of the 2023 ACM SIGSAC Conference on Computer and Communications Security (CCS '23), November 26--30, 2023, Copenhagen, Denmark}\acmDOI{10.1145/3576915.3623100}
\acmISBN{979-8-4007-0050-7/23/11}

\makeatletter
\gdef\@copyrightpermission{
  \begin{minipage}{0.3\columnwidth}
   \href{https://creativecommons.org/licenses/by/4.0/}{\includegraphics[width=0.90\textwidth]{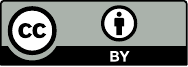}}
  \end{minipage}\hfill
  \begin{minipage}{0.7\columnwidth}
   \href{https://creativecommons.org/licenses/by/4.0/}{This work is licensed under a Creative Commons Attribution International 4.0 License.}
  \end{minipage}
  \vspace{5pt}
}
\makeatother

\author{Stefano Calzavara}
\affiliation{\institution{Università Ca’ Foscari Venezia} \country{}}
\email{stefano.calzavara@unive.it}

\author{Lorenzo Cazzaro}
\affiliation{\institution{Università Ca’ Foscari Venezia} \country{}}
\email{lorenzo.cazzaro@unive.it}

\author{Giulio Ermanno Pibiri}
\affiliation{\institution{Università Ca’Foscari Venezia} \country{}}
\email{giulioermanno.pibiri@unive.it}

\author{Nicola Prezza}
\affiliation{\institution{Università Ca’ Foscari Venezia} \country{}}
\email{nicola.prezza@unive.it}

\begin{document}

\begin{abstract}
Verifying the robustness of machine learning models against evasion attacks at test time is an important research problem. Unfortunately, prior work established that this problem is NP-hard for decision tree ensembles, hence bound to be intractable for specific inputs. In this paper, we identify a restricted class of decision tree ensembles, called \emph{large-spread} ensembles, which admit a security verification algorithm running in polynomial time. We then propose a new approach called \emph{verifiable learning}, which advocates the training of such restricted model classes which are amenable for efficient verification. We show the benefits of this idea by designing a new training algorithm that automatically learns a large-spread decision tree ensemble from labelled data, thus enabling its security verification in polynomial time. Experimental results on public datasets confirm that large-spread ensembles trained using our algorithm can be verified in a matter of seconds, using standard commercial hardware. Moreover, large-spread ensembles are more robust than traditional ensembles against evasion attacks, at the cost of an \revise{acceptable} loss of accuracy in the non-adversarial setting.
\end{abstract}


  
\maketitle

\section{Introduction}
Machine learning (ML) is now phenomenally popular and found an incredible number of applications. The more ML becomes pervasive and applied to critical tasks, however, the more it becomes important to verify whether automatically trained ML models satisfy desirable properties. This is particularly relevant in the security setting, where models trained using traditional learning algorithms proved vulnerable to \emph{evasion attacks}, i.e., malicious perturbations of inputs designed to force mispredictions at test time~\cite{BiggioCMNSLGR13,SzegedyZSBEGF13,DemetrioCBLAR21}. 

Unfortunately, verifying the security of ML models against evasion attacks is a computationally hard problem, because verification must account for all the possible malicious perturbations that the attacker may perform. In this work, we are concerned about the security of \emph{decision tree ensembles}~\cite{BreimanFOS84}, a well-known class of ML models particularly popular for non-perceptual classification tasks, which already received significant attention by the research community. Kantchelian et al.~\cite{KantchelianTJ16} first proved that the problem of verifying security against evasion attacks for decision tree ensembles is NP-complete when malicious perturbations are modeled by an arbitrary $L_p$-norm. In more recent work, Wang et al.~\cite{WangZCBH20} further investigated the problem and observed that the existing negative result largely generalizes to the apparently simpler case of decision stump ensembles, i.e., ensembles including just trees of depth one. They thus proposed \emph{incomplete} verification approaches for decision tree and decision stump ensembles, which can formally prove the absence of evasion attacks, but may incorrectly report evasion attacks also for secure inputs. This conservative approach is efficient and provides formal security proofs, however it is approximated and can draw a pessimistic picture of the actual security guarantees provided by the ML model. Complete verification approaches against specific attackers, e.g., modeled in terms of the $L_\infty$-norm, have also been proposed~\cite{ChenZS0BH19,RanzatoZ20}. They proved to be reasonably efficient in practice for many cases, however they have to deal with the NP-hardness of security verification, hence they are inherently bound to fail in the general setting, especially when the size of the decision tree ensembles increases. As a matter of fact, prior experimental evaluations show that security verification does not always terminate within reasonable time and memory bounds, leading to approximated estimates of the actual robustness of the decision tree ensemble against evasion attacks.

\subsubsection*{Contributions}
We propose a novel approach to the security verification of decision tree ensembles, called \emph{verifiable learning}. Our key idea is moving away from the intractable verification problems arising from arbitrary models to rather focus on learning restricted model classes designed to be easily verifiable. In particular:
\begin{enumerate}
    \item We identify a restricted class of decision tree ensembles, called \emph{large-spread} ensembles, which admit a security verification algorithm running in polynomial time for evasion attacks modeled in terms of an arbitrary $L_p$-norm, thus moving away from existing NP-hardness results (Section~\ref{sec:verification-algorithm}).
    
    \item We propose a new training algorithm that automatically learns a large-spread decision tree ensemble amenable for efficient security verification. In short, our algorithm first trains a traditional decision tree ensemble and then prunes it to satisfy the proposed large-spread condition (Section~\ref{sec:training}).
    
    \item We implement our training algorithm and experimentally verify its effectiveness on four public datasets. Our large-spread ensembles are more robust than traditional ensembles against evasion attacks and admit a much more efficient security verification, at the cost of just an \revise{acceptable} loss of accuracy in the non-adversarial setting (Section~\ref{sec:experiments}).
\end{enumerate}

\subsubsection*{Code availability}
We make our code available online (\url{https://github.com/LorenzoCazzaro/Verifiable-Learning-Robust-Tree-Ensembles}).

\section{Background}
In this section we review a few notions required to
appreciate the rest of the paper.
To improve readability, we summarize the main notation used in this paper in Table~\ref{tab:notation}. 

\begin{table}[t]
    \centering
    \begin{tabular}{c|l}
    \toprule
    $\vec{x},\vec{z}$ &  Instances drawn from the feature space $\feats$ \\
    $x_i$ & $i$-th component of the vector $\vec{x}$ \\
    $y$ & Class label drawn from the set of labels $\labels$ \\
    $d$ & Number of features of $\vec{x}$ (i.e., dimensionality of $\feats$) \\
    $t$ & Decision tree \\
    $n$ & Number of nodes of a decision tree \\
    $T$ & Tree ensemble \\
    $N$ & Number of nodes of a tree ensemble \\
    $m$ & Number of trees of a tree ensemble \\
    $\vec{\delta}$ & Adversarial perturbation \\
    $\Delta$ & Norm of an adversarial perturbation \\
    $A_{p,k}$ & Attacker based on $L_p$-norm (max. perturbation $k$) \\
    \bottomrule
    \end{tabular}
\caption{Summary of notation. In the definitions of $n$, $N$, $m$ we assume that the decision tree and tree ensemble we are predicating upon are clear from the context.}
\label{tab:notation}
\end{table}

\subsection{Supervised Learning}
Let $\feats \subseteq \R^d$ be a $d$-dimensional vector space of real-valued \textit{features}. An \emph{instance} $\vec{x} \in \feats$ is a $d$-dimensional feature vector $\langle x_1, x_2, \ldots, x_d \rangle$ representing an object in the vector space $\feats$. Each instance is assigned a class label $y \in \labels$ by an unknown \emph{target} function $f: \feats \rightarrow \labels$. As common in the literature, we focus on binary classification, i.e., we let $\labels = \{+1,-1\}$, because any multi-class classification problem can be encoded in terms of multiple binary classification problems.

Supervised learning algorithms automatically learn a \emph{classifier} $g: \feats \rightarrow \labels$ from a \emph{training set} of correctly labeled instances $\dtrain = \{(\vec{x}_i,f(\vec{x}_i))\}_i$, with the goal of approximating the target function $f$ as accurately as possible based on the empirical observations in the training set. The performance of classifiers is normally estimated on a \emph{test set} of correctly labeled instances $\dtest = \{(\vec{z}_i,f(\vec{z}_i))\}_i$, disjoint from the training set, yet drawn from the same data distribution. For example, the standard \emph{accuracy} measure $a(g,\dtest)$ counts the percentage of test instances where the classifier $g$ returns a correct prediction.

\subsection{Decision Trees and Tree Ensembles}
In this paper, we focus on traditional \emph{binary decision trees} for classification~\cite{BreimanFOS84}. Decision trees can be inductively defined as follows: a decision tree $t$ is either a leaf $\lambda(y)$ for some label $y \in \labels$ or an internal node $\sigma(f,v,t_l,t_r)$, where $f \in \{1,\ldots,d\}$ identifies a feature, $v \in \R$ is a threshold for the feature, and $t_l,t_r$ are decision trees (left and right child). We just write $\sigma(f,v)$ to represent an internal node when $t_l,t_r$ are unimportant. Decision trees are learned by initially putting all the training set into the root of the tree and by recursively splitting leaves (initially: the root) by identifying the threshold therein leading to the best split of the training data, e.g., the one with the highest information gain, thus transforming the split leaf into a new internal node.

At test time, the instance $\vec{x}$ traverses the tree $t$ until it reaches a leaf $\lambda(y)$, which returns the prediction $y$, denoted by $t(\vec{x}) = y$. Specifically, for each traversed tree node $\sigma(f,v,t_l,t_r)$, $\vec{x}$ falls into the left sub-tree $t_l$ if $x_f \leq v$, and into the right sub-tree $t_r$ otherwise. Fig.~\ref{fig:tree} represents an example decision tree of depth 2, which assigns label $+1$ to the instance $\langle 12,7 \rangle$ and label $-1$ to the instance $\langle 8,6 \rangle$. 

To improve their performance, decision trees are often combined into an \emph{ensemble} $T = \{t_1,\ldots,t_m\}$, which aggregates individual tree predictions, e.g., by performing majority voting. We write $T(\vec{x})$ for the prediction of $T$ on $\vec{x}$ and we let $N$ stand for the number of nodes of the ensemble $T$ when such ensemble is clear from the context. For simplicity, we focus on majority voting to aggregate individual tree predictions, assuming that the number of trees $m$ is odd to avoid ties. While ensembles trained using existing frameworks (like sklearn) may use more sophisticated aggregation techniques, our focus on large-spread ensembles trained using a custom algorithm gives us freedom on the choice of the aggregation strategy and majority voting already proves effective in practice. Notable ensemble methods include Random Forest~\cite{Breiman01} and Gradient Boosting~\cite{KeMFWCMYL17}.

\begin{figure}[t]
\centering
\begin{tikzpicture}[level 1/.style={sibling distance=4cm},level 2/.style={sibling distance=3cm}]
\tikzstyle{every node}=[rectangle,draw]
\node{$x_1 \leq 10$}
	child { node {$x_2 \leq 5$}
	        child { node {$+1$}}
	        child { node {$-1$}} }
	child { node {$x_2 \leq 8$}
		    child { node {$+1$}}
	        child { node {$-1$}} }
;
\end{tikzpicture}
\caption{Example of decision tree.}
\label{fig:tree}
\end{figure}
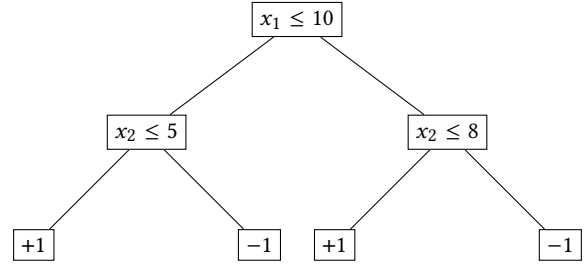

\subsection{Robustness}
Classifiers deployed in adversarial settings may be susceptible to \emph{evasion attacks}, i.e., malicious perturbations of test instances crafted to force prediction errors~\cite{BiggioCMNSLGR13,SzegedyZSBEGF13}. To capture this problem, the \emph{robustness} measure has been introduced~\cite{MadryMSTV18}. Below, we follow the presentation in~\cite{RanzatoZ20}.

An \emph{attacker} $A: \feats \rightarrow 2^{\feats}$ is modeled as a function from instances to sets of instances, i.e., $A(\vec{x})$ represents the set of all the adversarial manipulations of the instance $\vec{x}$, corresponding to the possible evasion attack attempts against $\vec{x}$. The \emph{stability} property requires that the classifier does not change its original prediction on some input for all its possible adversarial manipulations.

\begin{definition}[Stability]
The classifier $g$ is \emph{stable} on $\vec{x}$ for the attacker $A$ iff for all $\vec{z} \in A(\vec{x})$ we have $g(\vec{z}) = g(\vec{x})$.
\end{definition}

Stability is certainly a desirable property for classifiers deployed in adversarial settings;
however, a classifier that always predicts the same class for all the instances trivially satisfies stability for all the attackers,
but it is useless in practice because it lacks any predictive power.
Robustness improves upon stability by requiring the classifier to also perform correct predictions.

\begin{definition}[Robustness]
The classifier $g$ is \emph{robust} on $\vec{x}$ for the attacker $A$ iff $g(\vec{x}) = f(\vec{x})$ and $g$ is stable on $\vec{x}$ for $A$.
\end{definition}

Based on the definition of robustness, for a given attacker $A$, we can define the robustness measure $r_A(g,\dtest)$ by computing the percentage of test instances where the classifier $g$ is robust.

In the following, we focus on attackers represented in terms of an arbitrary $L_p$-norm, i.e., the attacker's capabilities are defined by some $p \in \N \cup \{0,\infty\}$ and the maximum perturbation $k$. For fixed $p$ and $k$, we assume the attacker $A_{p,k}(\vec{x}) = \{\vec{z} \in \feats ~|~ ||\vec{z} - \vec{x}||_p \leq k\}$.
\section{Efficient Robustness Verification}
\label{sec:verification-algorithm}

We first review results regarding the robustness verification problem
for single decision trees (Section~\ref{sec:one-tree}). We then generalize the result
to $m$ trees by introducing \emph{large-spread} decision tree ensembles (Section~\ref{sec:many-tree}),
which enable robustness verification in $O(N + m \log m)$ time. This is a major improvement over traditional decision tree ensembles, for which robustness verification is NP-complete~\cite{KantchelianTJ16}.

\subsection{Decision Trees}\label{sec:one-tree}

The robustness verification problem can be solved in $O(nd)$ time for a decision tree
with $n$ nodes when the attacker is expressed in terms of an arbitrary $L_p$-norm~\cite{WangZCBH20}.
This generalizes a previous result for the $L_\infty$-norm~\cite{ChenZS0BH19}.
The key idea of the algorithm is that stability on the instance $\vec{x}$
can be verified by identifying 
all the leaves that are reachable as the result of an evasion attack attempt $\vec{z} \in A_{p,k}(\vec{x})$; hence, stability holds iff all such leaves predict the same class.
This set of leaves can be computed by means of a simple tree traversal.
Correspondingly, assuming that $\vec{x}$ has label $y$, a decision tree $t$ is robust
on $\vec{x}$ iff $t(\vec{x}) = y$ and there does not exist any reachable leaf assigning
to $\vec{x}$ a label different from $y$.
The algorithm operates in two steps: (1) tree annotation and (2) robustness verification.


\subsubsection{Step 1 -- Tree Annotation}
The first step of the algorithm is a pre-processing operation -- performed only once --
where each node of the decision tree is annotated with auxiliary information for the second step.
The annotations are hyper-rectangles that symbolically represent
the set of instances which may traverse the nodes upon prediction.
The algorithm first annotates the root with the $d$-dimensional hyper-rectangle $(-\infty,+\infty]^d$, meaning that every instance will traverse the root. Children are then annotated by means of a recursive tree traversal: concretely, if the father node $\sigma(f,v,t_1,t_2)$ is annotated with $(l_1,r_1] \times \ldots \times (l_d,r_d]$, then the annotations of the roots of $t_1$ and $t_2$ are defined as $(l_1^1,r_1^1] \times \ldots \times (l_d^1,r_d^1]$ and $(l_1^2,r_1^2] \times \ldots \times (l_d^2,r_d^2]$ respectively, where:
\begin{equation}\label{eq:update-left}
(l_i^1,r_i^1] = \begin{cases}
(l_i,r_i] \cap (-\infty,v] = (l_i,\min\{r_i,v\}] & \textnormal{if } i = f \\
(l_i,r_i] & \textnormal{otherwise},
\end{cases}
\end{equation}
and:
\begin{equation}\label{eq:update-right}
(l_i^2,r_i^2] = \begin{cases}
(l_i,r_i] \cap (v,+\infty) = (\max\{l_i,v\},r_i] & \textnormal{if } i = f \\
(l_i,r_i] & \textnormal{otherwise}.
\end{cases}
\end{equation}
The annotation process terminates when all the nodes have been annotated.
Note that the complexity of this annotation step is $O(nd)$,
because all $n$ nodes are traversed and annotated with a hyper-rectangle of size $d$.

\subsubsection{Step 2 -- Robustness Verification}
Given an annotated decision tree and an instance $\vec{x}$, it is possible to identify the set of leaves which may be reached by $\vec{x}$ upon prediction in presence of adversarial manipulations.

Let $H = (l_1,r_1] \times \ldots \times (l_d,r_d]$ be the hyper-rectangle annotating a leaf $\lambda(y')$.
The minimal perturbation required to push $\vec{x}$ into $\lambda(y')$ is $\textit{dist}(\vec{x},H) = \vec{\delta} \in \R^d$, where:\footnote{We write $l_i - x_i + \varepsilon$ to stand for the minimum floating point number which is greater than $l_i - x_i$. The original paper~\cite{ChenZS0BH19} uses $l_i - x_i$ rather than $l_i - x_i + \varepsilon$, but this is incorrect because $\vec{\delta}$ identifies the minimal perturbation such that $\vec{x} + \vec{\delta} \in H$, however $x_i + l_i - x_i = l_i \not\in (l_i,r_i]$. We also assume here that $H$ is not empty, i.e., there does not exist any $(l_j,r_j]$ in $H'$ such that $l_j \geq r_j$.}
\begin{equation}\label{eq:lambda}
\delta_i = \var{dist}(\vec{x},H_i)=
\begin{cases}
0& \textnormal{if } x_i \in H_i=(l_i, r_i] \\
l_i - x_i + \varepsilon & \textnormal{if } x_i \leq l_i \\
r_i - x_i & \textnormal{if } x_i > r_i.
\end{cases}
\end{equation}
Thus, given the instance $\vec{x}$ with label $y$, it is possible to compute the set:
\begin{align}\label{eq:L}
    D = \Big\{ ||\vec{\delta}||_p \, | \,  \exists \, & H: \var{dist}(\vec{x},H) = \vec{\delta} 
        \wedge  ||\vec{\delta}||_p \leq k  \nonumber \\
        \wedge \, & H \textnormal{ annotates a leaf } \lambda(y') \textnormal{ with } y' \neq y \Big\}.
\end{align}

In other words, during the visit we find the leaves with a wrong class
where $\vec{x}$ might fall as the result of adversarial manipulations by the attacker $A_{p,k}$
and we compute the norms $||\vec{\delta}||_p$ of the
minimal perturbations $\vec{\delta}$ to be applied to $\vec{x}$ to push it there.
Hence, the tree is robust against the attacker $A_{p,k}$ iff $D = \emptyset$.
This computation can be performed in $O(nd)$ time,
since we have $O(n)$ leaves and each vector $\vec{\delta}$ with its norm can be computed in $\Theta(d)$ time.

\subsection{Generalization to Tree Ensembles}\label{sec:many-tree}

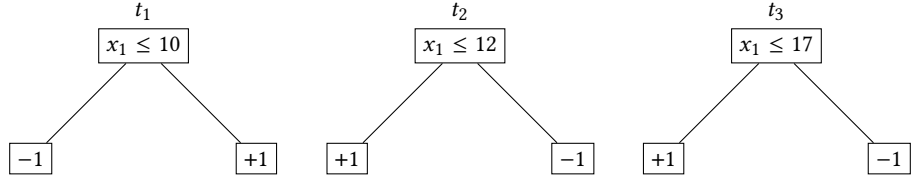
\begin{figure*}[t]
\centering
\begin{tikzpicture}[level 1/.style={sibling distance=3cm},level 2/.style={sibling distance=3cm}]
\tikzstyle{every node}=[rectangle,draw]
\node (T1)[label=$t_1$] {$x_1 \leq 10$}
	child { node {$-1$} }
	child { node {$+1$} }
;
\node (T2)[label=$t_2$,right= 3cm of T1]{$x_1 \leq 12$}
	child { node {$+1$} }
	child { node {$-1$} }
;
\node[label=$t_3$,right= 3cm of T2]{$x_1 \leq 17$}
	child { node {$+1$} }
	child { node {$-1$} }
;
\end{tikzpicture}
\caption{Example of tree ensemble with three decision trees.}
\label{fig:ensemble}
\end{figure*}

The robustness verification problem is NP-complete for tree ensembles
when the attacker is expressed in terms of an arbitrary $L_p$-norm~\cite{KantchelianTJ16}.
Of course, this negative result predicates over \emph{arbitrary} tree ensembles,
but does not exclude the possibility that restricted classes of ensembles
may admit a more efficient robustness verification algorithm.
In this section we introduce the class of \emph{large-spread} tree ensembles, which rule out the key source of complexity from the robustness verification problem and allow robustness verification in
$O(N + m \log m)$ time.

\subsubsection{Key Intuitions}
The key idea of the proposed large-spread condition allows one to verify the robustness guarantees of the individual decision trees in the ensembles
and \emph{compose their results} to draw conclusions about the robustness of the whole ensemble.

To understand why composing robustness verification results is unfeasible for arbitrary ensembles, consider the ensemble $T$ in Fig.~\ref{fig:ensemble} and an instance $\vec{x}$ with label $+1$ such that $x_1 = 11$. Consider the attacker $A_{1,2}$ who can modify feature 1 of at most $\pm 2$, then for every adversarial manipulation $\vec{z} \in A_{1,2}(\vec{x})$ we have $z_1 \in [9,13]$. We observe that the trees $t_1$ and $t_2$ are not robust on $\vec{x}$, because there exists an adversarial manipulation that forces them to predict the wrong class $-1$. However, the whole ensemble $T$ is robust on $\vec{x}$, because $T(\vec{x}) = +1$ and for every adversarial manipulation $\vec{z} \in A_{1,2}(\vec{x})$ we have $T(\vec{z}) = +1$, because either $t_1$ or $t_2$ alone is affected by the attack, hence at least two out of the three trees in the ensemble always perform the correct prediction. The example is deliberately simple to show that attacks against two different trees might be \emph{incompatible}, i.e., an attack working against one tree does not necessarily work against the other tree and vice-versa. This implies that the combination of multiple non-robust trees can lead to the creation of a robust ensemble.



The key intuition enabling our compositional reasoning is that interactions among different trees are only possible when the thresholds therein are close enough to each other. Indeed, in our example we showed that there exists an instance $\vec{x}$ which can be successfully attacked in both $t_1$ and $t_2$, yet no attack succeeds against both trees at the same time. The reason why this happens is that the thresholds in the roots of the trees (10 and 12 respectively) are too close to each other when taking into account the possible adversarial manipulations: an adversarial manipulation can corrupt the original feature value 11 to produce an arbitrary value in the interval $[9,13]$, which suffices to enable attacks in both $t_1$ and $t_2$. However, none of the attacks against $t_1$ works against $t_2$ and vice-versa. Conversely, it is not possible to find any instance $\vec{x}$ which can be attacked in both $t_2$ and $t_3$, because for every adversarial manipulation $\vec{z} \in A_{1,2}(\vec{x})$ we have $z_1 \in [x_1-2,x_1+2]$ and the distance between the thresholds in the trees ($17 - 12 = 5 > 4$) is large enough to ensure that the problem of incompatible attacks cannot exist, because the feature 1 can be attacked just in one of the two trees. For example, if $x_1 = 14$, then only $t_2$ can be attacked, while if $x_1 = 16$ only $t_3$ can be attacked; if $x_1 = 15$, instead, neither $t_2$ nor $t_3$ can be attacked.

\subsubsection{Large-Spread Ensembles}
We formalize this intuition by defining the \emph{$p$-spread} of a tree ensemble $T$ as the minimum distance between the thresholds of the same feature across different trees, according to the $L_p$-norm. If $\spread(T) > 2k$, where $k$ is the maximum adversarial perturbation, we say that $T$ is large-spread.

\begin{definition}[Large-Spread Ensemble]
\label{def:spread}
Given the ensemble $T = \{t_1, \dots, t_m\}$, its $p$-spread $\spread(T)$ is:
$$
\spread(T) = \min \bigcup_{\substack{1 \leq f \leq d\\t, t' \in T, t\neq t'}} \Big\{ ||v-v'||_p : \sigma(f,v) \in t \wedge \sigma(f,v') \in t' \Big\}.
$$
We say that $T$ is \emph{large-spread} for the attacker $A_{p,k}$ iff $\spread(T) > 2k$.
\end{definition}

A large-spread ensemble $T$ allows one to compose attacks working against individual trees to produce an attack against the ensemble as follows. Assuming $\vec{z}_i = \vec{x} + \vec{\delta}_i$ is an attack against a tree $t_i \in T$ and $\vec{z}_j = \vec{x} + \vec{\delta}_j$ is an attack against a different tree $t_j \in T$,
then the large-spread condition guarantees 
that $\vec{\delta}_i$ and $\vec{\delta}_j$ target disjoint sets of features,
i.e., 
they are orthogonal ($\vec{\delta}_i \cdot \vec{\delta}_j = 0$).
Indeed, each feature can be corrupted of $k$ at most, however the same feature can be reused in different trees only if the corresponding thresholds are more than $2k$ away, hence it is impossible for any feature value to traverse more than one threshold as the result of an evasion attack (we formalize and prove this result in 
\iffull Appendix~\ref{sec:proofs}). \else the full version~\cite{abs-2305-03626}).\fi
The disjointness condition of attacks implies that $\vec{z} = \vec{x} + \vec{\delta}_i + \vec{\delta}_j$ is an attack working against both $t_i$ and $t_j$ (assuming $||\vec{\delta}_i + \vec{\delta}_j||_p \leq k$), because $t_i(\vec{z})$ and $t_j(\vec{z})$ take the same prediction paths of $t_i(\vec{z}_i)$ and $t_j(\vec{z}_j)$ respectively, which are successful attacks against the two trees.
Note that this does not hold for arbitrary tree ensembles, like the one in Fig.~\ref{fig:ensemble}. Indeed, for that ensemble and an instance $\vec{x}$ such that $x_1 = 11$ the attack against $t_1$ subtracts 2 from the feature 1 and the attack against $t_2$ adds 2 to the feature 1, hence the sum of the two attacks would leave the instance $\vec{x}$ unchanged.

\subsubsection{Robustness Verification of Large-Spread Ensembles}
\label{sec:main-algo}
This compositionality result is powerful, because it allows
the efficient robustness verification of large-spread ensembles.
The intuition is that -- since the ensemble $T$ is large-spread --
the minimal perturbations $\{\vec{\delta}_i\}_i$
enabling attacks against the individual trees $\{t_i\}_i$ can be summed up together
to obtain a perturbation $\vec{\delta}$ enabling an attack against the whole ensemble.
More precisely, let $T' \subseteq T$ be the set of trees in $T$ which may suffer from a successful attack, then:
\begin{itemize}
    \item If $|T'| < \frac{m-1}{2}+1$, then the number of trees performing a wrong prediction under attack is too low to identify a successful attack against the whole ensemble.
    \item If $|T'| \geq \frac{m-1}{2}+1$, instead, we consider the $\frac{m-1}{2}+1$
attacks $\{\vec{\delta}_i\}_i$ with the \emph{smallest} $L_p$-norm.
An attack against $T$ is then possible iff $||\vec{\delta}||_p \leq k$, 
where $\vec{\delta}=\sum_{i=1}^{\frac{m-1}{2} + 1} \vec{\delta}_i$.
\end{itemize}

However, note that the complexity of this algorithm is $O(Nd + m \log m)$
because we annotate each of the $N$ nodes in the ensemble with a hyper-rectangle of size $d$
and we compute the minimum perturbations along with their norms,
as explained in Section~\ref{sec:one-tree}.
Moreover, to find the perturbations with the smallest norms,
we have to sort the pairs $(\vec{\delta_i},||\vec{\delta_i}||_p)$
in non-decreasing order of $L_p$-norm
in $O(m\log m)$ time.
We now show that the large-spread condition enables
a more efficient algorithm, running in $O(N + m \log m)$ time.

\subsubsection{Optimization}
If the minimal perturbations $\{ \vec{\delta}_i \}_i$
are pairwise orthogonal vectors, then the following facts hold.

\begin{fact}\label{fact:pairwise-orthogonality-0}
$|| \sum_{i=1}^q \vec{\delta}_i ||_0 = \sum_{i=1}^q || \vec{\delta}_i ||_0$,
if $\vec{\delta}_i \cdot \vec{\delta}_j = 0$, $\forall (i,j)$.
\end{fact}

\begin{fact}\label{fact:pairwise-orthogonality-inf}
$|| \sum_{i=1}^q \vec{\delta}_i ||_\infty = \underset{1 \leq i \leq q}{\max} \{|| \vec{\delta}_i ||_\infty\}$,
if $\vec{\delta}_i \cdot \vec{\delta}_j = 0$, $\forall (i,j)$.
\end{fact}

\begin{fact}\label{fact:pairwise-orthogonality}
$|| \sum_{i=1}^q \vec{\delta}_i ||_p = ( \sum_{i=1}^q || \vec{\delta}_i ||_p^p )^{1/p} $,
if $\vec{\delta}_i \cdot \vec{\delta}_j = 0$, $\forall (i,j)$.
\end{fact}

\iffull

Note that the proof of Fact~\ref{fact:pairwise-orthogonality-0} and Fact~\ref{fact:pairwise-orthogonality-inf}
is immediate, hence we just prove Fact~\ref{fact:pairwise-orthogonality} for the $L_p$-norm, with $p \in \mathbb{N}$.

\begin{proof}
We show the equivalence for $q=2$; the case for $q>2$ is a simple generalization.
By definition of $L_p$-norm, $|| \vec{\delta}_1 + \vec{\delta}_2 ||_p = ( \sum_{i=1}^d | \delta_{1,i} + \delta_{2,i} |^p )^{1/p}$.
The quantity $| \delta_{1,i} + \delta_{2,i} |^p$ is $\sum_{j=0}^p {p \choose j} | \delta_{1,i}^{p-j} \cdot  \delta_{2,i}^j|$ for the binomial theorem. Note that the latter sum can be rewritten as
$|\delta_{1,i}|^p + |\delta_{2,i}|^p + \sum_{j=1}^{p-1} {p \choose j} | \delta_{1,i}^{p-j} \cdot 
 \delta_{2,i}^j|$, where
$\sum_{j=1}^{p-1} {p \choose j} | \delta_{1,i}^{p-j} \delta_{2,i}^j| = 0$
because $\delta_{1,i} \cdot \delta_{2,i} = 0$ for any $1 \leq i \leq d$.
Therefore, $|| \vec{\delta}_1 + \vec{\delta}_2 ||_p = ( \sum_{i=1}^d |\delta_{1,i}|^p + \sum_{i=1}^d |\delta_{2,i}|^p )^{1/p} = ( ||\vec{\delta}_1||_p^p + ||\vec{\delta}_2||_p^p )^{1/p}$.

The generalization to an arbitrary number of vectors $q>2$ involves a multinomial theorem instead of a binomial theorem.
\end{proof}

\fi

We introduce the following operator to have a suitable way
of referring to the result of the three facts above:
\begin{equation}\label{eq:operator_oplus}
\bigoplus_{i=0}^q ||\vec{\delta_i}||_p =
\begin{cases}
\sum_{i=1}^q || \vec{\delta}_i ||_0, & \textnormal{if } p = 0 \\
\underset{1 \leq i \leq q}{\max} \{|| \vec{\delta}_i ||_\infty\} & \textnormal{if } p = \infty \\
( \sum_{i=1}^q || \vec{\delta}_i ||_p^p )^{1/p} & \textnormal{if } p \in \mathbb{N}.
\end{cases}
\end{equation}

Fact~\ref{fact:pairwise-orthogonality-0},~\ref{fact:pairwise-orthogonality-inf}, and~\ref{fact:pairwise-orthogonality} imply that we do not actually need to explicitly compute an adversarial perturbation if we just want its $L_p$-norm, which is exactly our case because we just need to check whether such norm does not exceed $k$. Since any adversarial perturbation against a large-spread ensemble results from the sum of pairwise orthogonal vectors, we can use Eq.~\ref{eq:operator_oplus} to compute the norm directly from the norms of the orthogonal vectors, i.e., the verification algorithm can operate on scalars rather than vectors, thus reducing its complexity by a $d$ factor.

In light of these considerations, we now revisit the tree traversal
from Section~\ref{sec:one-tree} to show that we can compute for each leaf of the tree
just a scalar $\Delta = ||\vec{\delta}||_p$, where 
$\vec{\delta} = \var{dist}(\vec{x},H)$ and $H$ is the hyper-rectangle which would
normally annotate the leaf.
Similarly to the linear-time tree visit described in~\cite{ChenZS0BH19} for the $L_\infty$-norm,
the idea is to maintain one \emph{global} hyper-rectangle during the visit
instead of one hyper-rectangle \emph{per node}.
Ultimately, this reduces the time complexity from $O(nd)$ to the optimal $O(n)$,
since the hyper-rectangle is not copied from parent to children.
The optimized variant of the algorithm is described in the {\sc Reachable} procedure of Algorithm \ref{alg:tree-annotation}.
This $O(n)$-time algorithm for arbitrary $L_p$-norm is, in fact, a combination of the $O(n)$-time algorithm of~\cite{ChenZS0BH19} (which  works only for the $L_\infty$-norm) with the generalization to any $L_p$-norm of~\cite{WangZCBH20} (which however runs in $O(nd)$ time).

\begin{algorithm}[t]
\caption{Optimized robustness verification algorithm for decision trees.}
\label{alg:tree-annotation}
\begin{algorithmic}[1]

\Function{Reachable}{$t,p,k,\vec{x},y$}
    \State{$H \gets (-\infty,+\infty]^d$}
    \State{$\Delta \gets 0$}
    \State{\Return $\Call{Traverse}{t,p,k,\vec{x},y,H,\Delta}$}
\EndFunction

\State{}

\Function{Traverse}{$t,p,k,\vec{x},y,H,\Delta$}
    \If{$t = \lambda(y')$}
        \If{$\Delta \leq k$ and $y' \neq y$}
            \State{\Return{$\{\Delta\}$}}
        \EndIf
        \State{\Return{$\emptyset$}}
    \EndIf

    \State{Let $t=\sigma(f,v,t_l,t_r)$}
    \State{$D \gets \emptyset$}
    \State{$H_f^* \gets H_f$} \Comment{copy}
    \State{$\delta_f = \var{dist}(\vec{x},H_f)$} \Comment{Eq.~\ref{eq:lambda}}
    \State{$H_f \gets H_f^* \cap (-\infty,v]$} \Comment{Eq.~\ref{eq:update-left}} 
    \State{$\delta'_f = \var{dist}(\vec{x},H_f)$} \Comment{Eq.~\ref{eq:lambda}}
    \State{$\Delta_l \gets \Call{Update-Norm}{p,\Delta,\delta_f,\delta'_f}$} \Comment{Eq.~\ref{eq:update-delta}} 
    \State{$D \gets D \cup \Call{Traverse}{t_l,p,k,\vec{x},y,H,\Delta_l}$} 
    \State{$H_f \gets H_f^* \cap (v,+\infty)$} \Comment{Eq.~\ref{eq:update-right}} 
    \State{$\delta'_f = \var{dist}(\vec{x},H_f)$} \Comment{Eq.~\ref{eq:lambda}}
    \State{$\Delta_r \gets \Call{Update-Norm}{p,\Delta,\delta_f,\delta'_f}$} \Comment{Eq.~\ref{eq:update-delta}} 
    \State{$D \gets D \cup \Call{Traverse}{t_r,p,k,\vec{x},y,H,\Delta_r}$} 
    \State{$H_f \gets H_f^*$} \Comment{Restore hyper-rectangle}
    \State{\Return $D$}
\EndFunction

\State{}

\Function{Robust-Tree}{$t,p,k,\vec{x},y$}
    \If{$t(\vec{x}) = y$} 
        \State{$D \gets \Call{Reachable}{t,p,k,\vec{x},y}$}
        \If{$D = \emptyset$}
            \State{\Return{True}}
        \EndIf
    \EndIf
    \State{\Return{False}}
\EndFunction
\end{algorithmic}
\end{algorithm}

We implement $H$ as an initially-empty map (e.g., using a hash table): $H_i \in \mathbb R^2$ is the entry associated to the $i$-th feature. If the map does not contain an entry for the $i$-th feature, then it is implicitly assumed $H_i= (-\infty,+\infty]$.
Let $H = (l_1,r_1] \times \ldots \times (l_d,r_d]$ be the 
state of the hyper-rectangle when
visiting
node $t = \sigma(f,v,t_1,t_2)$.
When moving to a child $t_j$ of $t$, with $j\in\{1,2\}$, note that the distance vector $\vec{\delta}$ changes only in its $f$-th component $\delta_f$, since only the $f$-th component $(l_f,r_f]$ of the hyper-rectangle $H$ changes.
We can therefore update $\Delta$ efficiently as follows.
Let $\Delta'$ and $H' = (l'_1,r'_1] \times \ldots \times (l'_d,r'_d]$ be the perturbation distance and hyper-rectangle associated to any of $t$'s children. 
Let $\delta'_f$ be the quantity defined in Eq.~\ref{eq:lambda}.
We extend the linear-time algorithm of \cite{ChenZS0BH19} to an arbitrary $L_p$-norm by noting that the following
is implied by Facts~\ref{fact:pairwise-orthogonality-0}, \ref{fact:pairwise-orthogonality-inf}, and \ref{fact:pairwise-orthogonality}:
\begin{equation}\label{eq:update-delta}\small
\Call{Update-Norm}{p,\Delta,\delta_f,\delta'_f} = 
\begin{cases}
\Delta - ||\delta_f||_0 + ||\delta'_f||_0 & \textnormal{if } p = 0 \\
\max\big(\Delta,  |\delta'_f|\big) & \textnormal{if } p = \infty \\
\big(\Delta^p - |\delta_f|^p + |\delta'_f|^p\big)^{1/p} & \textnormal{if } p \in \mathbb{N}.
\end{cases}
\end{equation}

By definition, it is clear that {\sc Update-Norm} is computed in $O(1)$ time.
The correctness of the case $p=\infty$ (as also discussed in \cite{ChenZS0BH19})
follows from the fact that it must be $|\delta'_f| \geq |\delta_f|$, since $(l'_i,r'_i] \subseteq (l_i,r_i]$.
In conclusion, we spend $O(1)$ time per node and the time complexity of the whole visit is therefore $O(n)$.
Hence, the set $D$ in Eq.~\ref{eq:L} is computed in $O(n)$ time rather than $O(nd)$ time
as we previously described in Section~\ref{sec:one-tree}.
This also lowers the time complexity of the robustness verification for decision trees
shown in the {\sc Robust-Tree} procedure of Algorithm~\ref{alg:tree-annotation} to just $O(n)$ rather than $O(nd)$. 
Since robustness verification for large-spread ensembles builds on the verification algorithm of
the individual trees therein, this optimization reduces the complexity of our final algorithm.

\subsubsection{Final Algorithm}
We conclude this section with Algorithm~\ref{alg:ensembles},
our robustness verification algorithm for large-spread ensembles, whose
correctness is stated in the following theorem and proved in \iffull Appendix~\ref{sec:proofs}. \else the full version~\cite{abs-2305-03626}. \fi
It follows the description in Section~\ref{sec:main-algo}, revised to operate with
norms (scalars) rather than vectors.
 
\begin{restatable}{theorem}{ensembles}
\label{thm:ensembles}
Let $\vec{x}$ be an instance with label $y$. A tree ensemble $T$ such that $\spread(T) > 2k$ is robust on $\vec{x}$ against the attacker $A_{p,k}$ iff $\textsc{Robust}(T,p,k,\vec{x},y)$ returns True.
\end{restatable}

Observe that the complexity of Algorithm~\ref{alg:ensembles} is $O(N + m \log m)$,
where $N$ and $m$ are, respectively, the total number of nodes and trees in the ensemble.
Verifying the robustness of the $m$ individual trees in the ensemble and updating vector $\vec{\Delta}$
takes $O(N)$ time thanks to the linear-time Algorithm~\ref{alg:tree-annotation}.
Afterwards, the algorithm sorts $\vec{\Delta}$ in $O(m\log m)$ time
and computes the minimum norm required to attack at least $\frac{m-1}{2}+1$ trees in $O(m)$ time.

\begin{algorithm}[t]
\caption{Robustness verification algorithm for large-spread tree ensembles.}
\label{alg:ensembles}
\begin{algorithmic}[1]
\Function{Robust}{$T,p,k,\vec{x},y$}
    \If{$T(\vec{x}) = y$}
        \State{\Return{$\Call{Stable}{T,p,k,\vec{x},y}$}}
    \EndIf
    \State{\Return{False}}
\EndFunction

\State{}

\Function{Stable}{$T, p, k, \vec{x},y$}
    \State{$\var{num\_unstable\_trees} \gets 0$}
    \State{$\vec{\Delta} \gets [+\infty,\ldots,+\infty]$} \Comment{Vector of size $m$}
    \For{$i \gets 1$ to $m$}
        \State{$D \gets \Call{Reachable}{t_i,p,k,\vec{x},y}$}
        \If{$D \neq \varnothing$}
            \State{$\Delta_i \gets \min D$}
            \State{$\var{num\_unstable\_trees} \gets \var{num\_unstable\_trees} + 1$}
        \EndIf
    \EndFor
    \If{$\var{num\_unstable\_trees} \geq (m-1)/2+1$}
        \State{Sort $\vec{\Delta}$ in non-decreasing order}
        \State{$\Delta = \bigoplus_{i=0}^{(m-1)/2+1} \Delta_i$} \Comment{Eq.~\ref{eq:operator_oplus}}
        \If{$\Delta \leq k$}
            \State{\Return{False}}
        \EndIf
    \EndIf
    \State{\Return{True}}
\EndFunction
\end{algorithmic}
\end{algorithm}
\section{Training Large-Spread Ensembles}
\label{sec:training}
We have described an efficient robustness verification algorithm for large-spread ensembles
in Section~\ref{sec:verification-algorithm}.
However, traditional decision tree ensembles trained using, e.g., sklearn, do not necessarily enjoy the large-spread condition. Here we discuss possible ideas for training algorithms designed to enforce the large-spread condition and we present a specific solution from the design space.

\subsection{Design Space}
While reasoning about the design of a training algorithm for large-spread ensembles, we considered different approaches falling in three broad classes:

\begin{enumerate}
    \item \emph{Custom ensemble learning algorithms}. Develop new learning algorithms in the spirit of Random Forest~\cite{Breiman01} or Gradient Boosting~\cite{KeMFWCMYL17}, designed to constrain the ensemble shape so as to satisfy the large-spread condition. For example, one might train each tree while taking into account the thresholds already present in the previously trained trees, to then remove the training data which might lead to learning thresholds which are too close to the existing ones. Indeed, recall that thresholds are learned from the training data, hence all the possible thresholds are known a priori.

    \item \emph{Training set partitioning}. Pre-compute a partition of the training data so that each decision tree in the ensemble is trained over highly separated instances, thus leading to an ensemble of trees satisfying the large-spread condition. The simplest instantiation of this idea would be partitioning the set of features and train different trees over different subsets of features, so that the large-spread condition is trivially satisfied, but more fine-grained strategies based on instance partitioning would also be feasible.

    \item \emph{Pruning techniques}. Train a standard decision tree ensemble, e.g., using the Random Forest algorithm, and prune it so as to keep only trees satisfying the large-spread condition. A variant of this technique might perform different types of mutations of the available trees to improve the effectiveness of pruning.
\end{enumerate}

Although we consider all these routes to be viable and worth investigating, in this work we decide to prioritize the third class of solutions. Compared to the first class, pruning leads to a range of simple and intuitive solutions, which take advantage of state-of-the-art implementations of existing training algorithms, e.g., those available in sklearn. This simplifies the deployment of an efficient and robust implementation. Moreover, pruning does not necessarily require a massive amount of training data and features, as needed for an effective training set partitioning (second class). In the last part of this section, we also discuss how to leverage feature partitioning to improve the effectiveness of our pruning-based learning algorithm in those settings where a high number of features is available (\emph{hierarchical training}).

\subsection{Proposed Training Algorithm}
\label{sec: training-algorithm}
Here we present our training algorithm. We motivate its design, describe how it works and discuss a few relevant aspects of the proposed solution.

\subsubsection{Preliminaries}
Our problem of interest can be formulated as follows: given a decision tree ensemble $T$ and a size $0 < s \leq |T|$, determine whether there exists an ensemble $T' \subseteq T$ such that $T'$ is large-spread and $|T'| = s$. We refer to this problem as the \emph{large-spread subset} problem for decision tree ensembles. Unfortunately, we can prove that this problem is NP-hard. The proof is provided in \iffull Appendix~\ref{sec:proofs2}. \else the full version~\cite{abs-2305-03626}. \fi

\begin{restatable}{theorem}{nphard}
\label{thm:np-hard}
The large-spread subset problem is NP-hard.
\end{restatable}

The theorem implies that it is computationally hard to train large-spread ensembles by pruning when the desired number of trees therein is enforced a priori, which is normally the case because the number of trees is a standard hyper-parameter of ensemble methods. One might argue that this negative result is not a showstopper, because training is performed just once and one might devise efficient heuristic approaches to approximate the large-spread subset problem, however preliminary experiments on public datasets suggest that any training approach which is \emph{purely} based on pruning is likely ineffective in practice. Indeed, we empirically observed on our datasets that traditional random forests trained using sklearn are not directly amenable for pruning, because any two trees in the ensemble already violate the large-spread condition when joined into an ensemble of size two. Our understanding of this phenomenon is that there exist some important features which are pervasively reused across different trees, which often learn the same thresholds, thus making the identification of a large-spread ensemble unfeasible. Our training algorithm thus integrates a greedy heuristic approach to pruning with a mutation operation, which perturbs thresholds so as to actively enforce the large-spread condition even when it would not be possible by pruning alone.

\subsubsection{Training Algorithm}
The proposed training algorithm takes as input a training set $\dtrain$, a number of trees $m$, a norm $p$ and a maximum perturbation $k$. \revise{In addition to the classic hyper-parameters of tree learning such as tree depth, the algorithm relies on a few specific hyper-parameters: a maximum number of iterations $MAX\_ITER \in \N$, a multiplicative factor $MULT \in \N$ and a real-valued interval $INTV \in \R \times \R$.} From a high level point of view, the algorithm operates by training a standard random forest $T$ including $\revise{MULT \cdot m}$ trees to then select a set of $m$ trees constituting a large-spread ensemble $T^*$. This is done by a combination of pruning and mutation of the trees in $T$. After picking a random tree of $T$ to begin with, the algorithm iteratively \writtenbyLC{tries to identify the other $m-1$ trees by means of a greedy approach. The candidate tree $t$ to be inserted in $T^*$ is always the tree in $T$ minimizing the number of \emph{feature overlaps} with $T^*$, i.e., the number of features violating the large-spread condition in $T^* \cup \{t\}$}. If the number of feature overlaps is greater than zero, the ensemble is fixed to enforce the large-spread condition by iteratively removing the overlaps. In particular, let $\sigma(f,v)$ and $\sigma(f,v')$ be two nodes from different trees such that $||v-v'||_p \leq 2k$. We sample a perturbation $\delta \in \revise{INTV}$, we subtract $\delta$ from $\min(v,v')$ and we sum $\delta$ to $\max(v,v')$ in the attempt to fix the overlap. Since this change might introduce new overlaps, we then iterate through the ensemble until all the overlaps have been fixed (i.e., the ensemble is large-spread) or the maximum number of iterations \revise{$MAX\_ITER$} have been reached. \writtenbyLC{If all the overlaps of $T^* \cup \{t\}$ have been fixed, i.e., the resulting tree-based ensemble is large-spread, then the extended large-spread ensemble becomes the new large-spread ensemble $T^*$, otherwise $T^*$ is not extended and the tree $t$ is discarded. Then the algorithm tries to extend $T^*$ with another tree in $T$, unless $T^*$ has reached the desired number of trees or all the trees in $T$ have been selected for extending the large-spread ensemble.} The pseudocode of the training algorithm is presented in Algorithm~\ref{alg:training}.

\begin{algorithm}[t]
\small
\caption{Training algorithm for large-spread ensembles.}
\label{alg:training}
\begin{algorithmic}[1]
\Require{Hyper-param. $MAX\_ITER \in \N, MULT \in \N, INTV \in \R \times \R$}
\Function{TrainLargeSpread}{$\dtrain,m,p,k$}
\State{$T \gets \Call{TrainRandomForest}{\dtrain,MULT \cdot m}$}  
\State{$t \gets \Call{SampleTree}{T}$}                  \Comment{Choose a random tree from $T$}
\State{$T \gets T \setminus \{t\}$}
\State{$T^* \gets \{t\}$}
\State{$i \gets 1$}
\While{$i < MULT \cdot m$ and $|T^*| < m$}
    \State{$i \gets i+1$}
    \State{$t \gets \Call{GetBestTree}{T,T^*,p,k}$}
    \State{$T \gets T \setminus \{t\}$}
    \State{$\overline{T^*} \gets T^* \cup \{t\}$}
    \State{$\overline{T^*} \gets \Call{FixForest}{\overline{T^*},p,k}$}
    \If{$\overline{T^*} \neq \bot$}                                   \Comment{{\sc FixForest} succeded}
        \State{$T^* \gets \overline{T^*}$}
    \EndIf
\EndWhile
\If{$|T^*| = m$} \Comment{{\sc TrainLargeSpread} succeded}
    \State{\Return{$T^*$}}
\Else
    \State{\Return{$\bot$}}
\EndIf
\EndFunction

\State{}

\Function{GetBestTree}{$T,T^*,p,k$}
\State{$t^* \gets \bot$}
\State{$min\_f\_overlaps \gets +\infty$}
\For{$t \in T$}
    \State{$overlaps \gets \{\sigma(f,v) \in t ~|~ \exists \sigma(f,v') \in T^*: ||v-v'||_p \leq 2k\}$}
    \State{$f\_ov \gets |\{f ~|~ \exists \sigma(f,v) \in overlaps\}|$}
    \If{$f\_ov < min\_f\_overlaps$}
        \State{$min\_f\_overlaps \gets f\_ov$}
        \State{$t^* \gets t$}
    \EndIf
\EndFor
\State{\Return{$t^*$}}
\EndFunction

\State{}

\Function{FixForest}{$T^*,p,k$}
\State{$iter \gets 0$}
\While{$T^*$ is not large-spread and $iter < MAX\_ITER$}
    \State{$iter \gets iter + 1$}
    \For{$t, t' \in T^*,\ \sigma(f,v)\in t,\ \sigma(f',v')\in t'$ }
        \If{$f = f'$ and $||v-v'||_p \leq 2k$}
            \State{$\delta \gets \Call{Random}{INTV}$}    \Comment{Sample a float in $INTV$}
            \If{$v \leq v'$}
                \State{$\sigma(f,v) \gets \sigma(f,v-\delta)$}
                \State{$\sigma(f',v') \gets \sigma(f',v'+\delta)$}
            \Else
                \State{$\sigma(f,v) \gets \sigma(f,v+\delta)$}
                \State{$\sigma(f',v') \gets \sigma(f',v'-\delta)$}
            \EndIf
        \EndIf
    \EndFor
\EndWhile
\If{$T^*$ is not large-spread}
    \State{\Return{$\bot$}}
\EndIf
\State{\Return{$T^*$}}
\EndFunction
\end{algorithmic}
\end{algorithm}


\subsubsection{Complexity}
\revise{Recall that each tree has at most $n$ nodes and we fix $MULT$ to be a small constant, e.g., $MULT \in [2,6]$.}
{\sc TrainLargeSpread} calls $O(m)$ times {\sc GetBestTree}  and {\sc FixForest}. The former function {\sc GetBestTree} iterates at most $|T|\in O(m)$ times ($t\in T$, line 23) the construction of set \emph{overlaps}. A naive way of building this set is to iterate over all nodes of $t$ (at most $n$ nodes) and compare their thresholds with all the thresholds appearing in the nodes of $T^*$ (at most $mn$ nodes), leading to time $O(mn^2)$ to build one instance of \emph{overlaps}. We observe that it is easy to speed up this step using balanced search trees, but we leave optimizations to further extensions of this work. To sum up, {\sc GetBestTree} takes $O(m^2n^2)$ and, hence, the
$O(m)$ calls to {\sc GetBestTree} cost overall time $O(m^3n^2)$.
Function {\sc FixForest} iterates $MAX\_ITER$ times the \texttt{for} loop at line 35.
Each iteration of the \texttt{for} loop costs $O(1)$ time and there are at most $m^2n^2$ iterations
because the loop iterates over all possible combinations of $\sigma(f,v)$ and $\sigma(f',v')$ belonging to two distinct trees of $T^*$. Since there are at most $mn$ nodes in $T^*$, the number of iterations is at most $m^2n^2$.
To this cost, we have to add the $MAX\_ITER$ evaluations of "$T^*$ is not large-spread" (line 33); this predicate can be evaluated in $O(|T^*|^2) = O(m^2n^2)$ time by comparing all pairs of thresholds appearing in $T^*$.
We conclude that the running time of the $O(m)$ iterations of {\sc FixForest} is in total $O(MAX\_ITER \cdot m^3n^2)$. This dominates the running time of the $O(m)$ iterations of {\sc GetBestTree}, so we conclude that $O(MAX\_ITER \cdot m^3n^2)$ is also the running time of {\sc TrainLargeSpread}. This cost is paid in addition to the cost of training the standard random forest at line 2.

As noted above, although it is feasible to reduce this complexity using appropriate data structures, we observe that $(i)$ training is often performed only once, so any optimization just offers limited benefits and is left to future work, and $(ii)$ the number of trees and nodes is often small enough to make a cubic complexity acceptable in practice. As a matter of fact, \revise{our experimental evaluation gives evidence about the acceptable empirical efficiency of the proposed training algorithm.}

\subsubsection{Hierarchical Training}
\label{sec: hierarchical-training}
We observe that our training algorithm can fail, in particular when it is not possible to add one tree to the current large-spread ensemble and reduce to zero the overlaps by our mutation routine, i.e., the number of overlaps resulting from adding a tree to the large-spread ensemble is too high. However, we show in our experimental evaluation (see Section~\ref{sec:experiments}) that it is possible to train large-spread ensembles of different dimensions after some parameter tuning. In particular, we propose an intuitive and effective technique to mitigate the risks of failures during training. A key insight is that the larger the ensemble is, the more difficult it becomes to avoid violations of the large-spread requirement, because ensembles including many trees also have many thresholds, hence overlaps become harder to avoid. We thus propose a \emph{hierarchical} training approach as follows:
\begin{enumerate}
    \item We first partition the set of features in $l$ disjoint subsets and we build $l$ different projections of the training set $\dtrain$, based on such feature sets.
    \item We train a large-spread ensemble of size $\frac{m}{l}$ on each of the $l$ different training sets using Algorithm~\ref{alg:training} and we finally merge all the trained ensembles into an ensemble of $m$ trees. 
\end{enumerate}

Note that the final ensemble is indeed large-spread, because each of the merged ensembles ensures the large-spread condition on the trees therein, and trees from different ensembles cannot violate the large-spread condition because they are built on disjoint sets of features. For example, an ensemble of 100 trees can be trained by building 4 disjoint projections of the training data (based on feature partitioning) and training an ensemble of 25 trees on each of them. We empirically observed that this approach may improve the effectiveness of the training process, by enabling the construction of larger ensembles in practice. We report on experiments confirming this observation in the next section.

\section{Experimental Evaluation}
\label{sec:experiments}

\subsection{Experimental Setup}
To show the practical relevance of our theory, we develop two tools on top of it and we prove their effectiveness on public datasets.

\subsubsection{Tools}
Our first tool CARVE\footnote{CARVE - CompositionAl Robustness Verifier for tree Ensembles} is a C++ implementation of the proposed robustness verification algorithm for large-spread ensembles (Algorithm~\ref{alg:ensembles}). It takes as input a random forest classifier $T$, a norm $p$, a maximum perturbation $k$ and a test set $\dtest$ to return as output the robustness score $r_{A_{p,k}}(T,\dtest)$. CARVE assumes that $T$ is large-spread and implements majority voting as the aggregation scheme of individual tree predictions. Our second tool LSE is a sequential Python implementation of the proposed training algorithm for large-spread ensembles (Algorithm~\ref{alg:training}). Starting from a training set $\dtrain$, a number of trees $m$, a norm $p$ and a maximum perturbation $k$, it returns a large-spread ensemble $T^*$ of $m$ trees (unless the training algorithm fails by returning $\bot$). 
The random forest trained before pruning is created using sklearn.

\begin{table}[t]
\caption{Dataset statistics.}
\label{tab:datasets}
    \centering
    \begin{tabular}{c|c|c|c}
    \toprule
    \textbf{Dataset} & \textbf{Instances} & \textbf{Features} & \textbf{Distribution} \\
    \midrule
    Fashion-MNIST & 13,866 & 784 & $50\% / 50\%$ \\
    MNIST & 14,000 & 784 & $51\% / 49\%$ \\
    REWEMA & 6,271 & 630 & $50\% / 50\%$ \\
    \revise{Webspam} & \revise{350,000} & \revise{254} & $\revise{70\% / 30\%}$ \\
    \bottomrule
    \end{tabular}
\end{table}

\subsubsection{Methodology}
\label{sec:exp-methodology}
\writtenbyLC{Our experimental evaluation is performed on \revise{four} public datasets: Fashion-MNIST\footnote{\url{https://www.openml.org/search?type=data\&sort=runs\&id=40996&status=active}}, MNIST\footnote{\url{https://www.openml.org/search?type=data\&sort=runs\&id=554}}, REWEMA\footnote{\url{https://www.kaggle.com/code/kerneler/starter-rewema-c5ce57b7-e/input}} and \revise{Webspam}\footnote{\url{https://www.csie.ntu.edu.tw/~cjlin/libsvmtools/datasets/binary.html}}. Since Fashion-MNIST and MNIST are datasets associated to multiclass classification tasks and we focus on binary classification tasks in this work, we consider two subsets of them. In particular, for Fashion-MNIST we consider the instances with class 0 (T-shirt/top) and 3 (Dress), while for MNIST we keep the instances representing the digits 2 and 6. The key characteristics of the chosen datasets are reported in Table~\ref{tab:datasets}. 
\revise{The chosen datasets are representative for different reasons: Fashion-MNIST, MNIST and Webspam have already been considered in the robustness verification literature~\cite{Andriushchenko019, ChenZS0BH19, RanzatoZ20, WangZCBH20}; moreover, REWEMA and Webspam are associated with a security-relevant classification task (malware and spam detection, respectively) for which the robustness verification of the employed classifier is critically important.}
In general, we choose datasets with a high number of features, where it may be useful to train large tree ensembles to reach the best performance. Each dataset is partitioned into a training set and a test set, using 70/30 stratified random sampling.}


In our experimental evaluation we make use of two training algorithms to learn different types of classifiers: $(i)$ a majority-voting classifier based on a traditional random forest (RF) trained using sklearn, and $(ii)$ a majority-voting classifier based on a large-spread tree ensemble trained using LSE. \writtenbyLC{Moreover, we consider tree-based classifiers of different sizes: $(i)$ small ensembles with $25$ trees of maximum depth $4$; $(ii)$ large ensembles with $101$ trees with maximum depth $6$. We only consider ensembles with an odd number of trees in order to avoid ties in classification.}

Robustness verification is then performed using CARVE and SILVA, a state-of-the-art verifier for traditional decision tree ensembles based on abstract interpretation~\cite{RanzatoZ20}. Note that SILVA can be applied to arbitrary ensembles, while CARVE can only be used on large-spread ensembles. Since SILVA leverages the hyper-rectangle abstract domain for verification, which does not introduce any loss of precision for $L_\infty$-norm attackers but might lead to an over-approximation for generic $L_p$-norm attackers, we only focus on $L_\infty$-attackers in our \revise{comparison. For the sake of completeness, in our evaluation of CARVE we also consider robustness against $L_1$-attackers and $L_2$-attackers for large-spread ensembles.} 

\revise{Finally, in our evaluation we consider different perturbations $k \in \{0.0050, 0.0100,$ $ 0.0150\}$ for the MNIST, Fashion-MNIST and REWEMA datasets, while we assume $k \in \{0.0002, 0.0004, 0.0006\}$ for Webspam. 
We choose different perturbations for the Webspam dataset to be aligned with previous work and to obtain roughly the same decrease in robustness observed on the other three datasets for the considered tree-based classifiers. Indeed, Chen et. al.~\cite{ChenZBH19} showed in their experimental evaluation that the certified minimum adversarial perturbation obtained for the Webspam dataset is one order of magnitude smaller than the one obtained for the MNIST dataset, i.e., models trained over Webspam would be too fragile to be usable when tested against larger perturbations.}

\subsubsection{LSE Setup}
\label{sec:lse-setup}
\writtenbyLC{Our tool LSE requires the user to specify the value of some additional parameters (described in Section~\ref{sec: training-algorithm}) with respect to the traditional implementation of the training algorithm for random forests by sklearn. The norm $p$ and the perturbation $k$ depend on the assumed attacker's capabilities, so they do not require a particular tuning. Still, \revise{other parameters such as} the number of partitions $l$ for the hierarchical training and the maximum number of iterations $MAX\_ITER$ of the $\textproc{FixForest}$ procedure require some tuning. Indeed, although partitioning the features may enable the training of larger ensembles, a too high number of partitions might negatively affect the accuracy of the resulting large-spread ensemble, because each sub-forest has only a partial view on the set of available features and some patterns may not be learned. In the same way, the maximum number of rounds $MAX\_ITER$ has an impact on the success of the training procedure, since a minimum number of rounds is required to adjust the thresholds of the ensemble, but a too high number of rounds may modify the thresholds too much and downgrade the predictive power of the model. We perform some experiments in order to assess the influence of these parameters on the success of training a large-spread ensemble and on the accuracy of the resulting model \revise{to then pick the best-performing models in our experimental evaluation}. For space reasons, we discuss details in \iffull Appendix~\ref{sec:appendix-tuning}. \else the full version~\cite{abs-2305-03626}. \fi}

\begin{table*}[t]
\caption{Accuracy and robustness measures for traditional and large-spread ensembles. Robustness is computed against $A_{\infty,k}$. We highlight in bold the cases in which the gap between the accuracy and the robustness of the traditional tree-based ensemble and large-spread ensemble is at least of $0.05$.}
\label{tab:measures}
    \centering
    \begin{tabular}{c|c|c|c|c|c|c|c}
    \toprule
    \multirow{2}{*}{\textbf{Dataset}} & \multirow{2}{*}{\textbf{$k$}} & \multirow{2}{*}{\textbf{Trees}} & \multirow{2}{*}{\textbf{Depth}} & \multicolumn{2}{c}{\textbf{Accuracy}} & \multicolumn{2}{c}{\textbf{Robustness}} \\
    \cmidrule{5-8} 
    & & & & Traditional & Large-Spread & Traditional & Large-Spread \\
    \midrule
    \multirow{6}{*}{Fashion-MNIST} & \multirow{2}{*}{$0.0050$} & 25 & 4 & 0.93 & \revise{0.92} & 0.90 & 0.90\\
    & & 101 & 6 & 0.96 & \revise{0.96} & 0.91 & \revise{0.93}\\
    \cline{2-8}
    & \multirow{2}{*}{$0.0100$} & 25 & 4 & 0.93 & \revise{0.92} & 0.86 & \revise{0.87} \\
    & & 101 & 6 & 0.96 & \revise{0.94} & \textbf{0.79} & \revise{\textbf{0.91}}\\
    \cline{2-8}
    & \multirow{2}{*}{$0.0150$} & 25 & 4 & 0.93  & 0.91 & \textbf{0.60} & \revise{\textbf{0.88}} \\
    & & 101 & 6 & 0.96 & \revise{0.92} & \textbf{0.51 $\pm$ 0.01} & \revise{\textbf{0.89}} \\
    \midrule
    \multirow{6}{*}{MNIST} & \multirow{2}{*}{$0.0050$} & 25 & 4 & 0.97 & \revise{0.97} & \textbf{0.90} & \revise{\textbf{0.96}} \\
    & & 101 & 6 & 0.99 &  0.99 & 0.94 & 0.97 \\
    \cline{2-8}
    & \multirow{2}{*}{$0.0100$} & 25 & 4 & 0.97 & \revise{0.97} & \textbf{0.72} & \revise{\textbf{0.90}} \\
    & & 101 & 6 & 0.99 & 0.99 & \textbf{0.77 $\pm$ 0.02} & \revise{\textbf{0.97}} \\
    \cline{2-8}
    & \multirow{2}{*}{$0.0150$} & 25 & 4 & 0.97 & \revise{0.97} & \textbf{0.64} & \revise{\textbf{0.83}}\\
    & & 101 & 6 & 0.99 & 0.99 & \textbf{0.67 $\pm$ 0.05} & \revise{\textbf{0.94}} \\
    \midrule
    \multirow{6}{*}{REWEMA} & \multirow{2}{*}{$0.0050$} & 25 & 4 & 0.88 & \revise{0.88} & 0.85 & \revise{0.87}\\
    & & 101 & 6 & \textbf{0.98} & \textbf{0.89} & 0.88 & 0.89 \\
    \cline{2-8}
    & \multirow{2}{*}{$0.0100$} & 25 & 4 & 0.88 & \revise{0.88} & 0.83 & \revise{0.87} \\
    & & 101 & 6 & \textbf{0.98} & \revise{\textbf{0.89}} & 0.86 & \revise{0.88} \\
    \cline{2-8}
    & \multirow{2}{*}{$0.0150$} & 25 & 4 & 0.88 & 0.88 & 0.83 & \revise{0.85}\\
    & & 101 & 6 & \textbf{0.98} & \revise{\textbf{0.88}} & \textbf{0.78} & \revise{\textbf{0.88}} \\
    \midrule
    \multirow{6}{*}{\revise{Webspam}} & \multirow{2}{*}{\revise{$0.0002$}} & \revise{25} & \revise{4} & \revise{0.90} & \revise{0.90} & \revise{0.83} & \revise{0.87}\\
    & & \revise{101} & \revise{6} & \revise{0.94} & \revise{0.91} & \revise{0.88} & \revise{0.90}\\
    \cline{2-8}
    & \multirow{2}{*}{\revise{$0.0004$}} & \revise{25} & \revise{4} & \revise{0.90} & \revise{0.89} & \revise{\textbf{0.80}} & \revise{\textbf{0.86}} \\
    & & \revise{101} & \revise{6} & \revise{\textbf{0.94}} & \revise{\textbf{0.89}} & \revise{0.85} & \revise{0.86}\\
    \cline{2-8}
    & \multirow{2}{*}{\revise{$0.0006$}} & \revise{25} & \revise{4} & \revise{0.90} & \revise{0.89} & \revise{\textbf{0.78}} & \revise{\textbf{0.85}} \\
    & & \revise{101} & \revise{6} & \revise{\textbf{0.94}} & \revise{\textbf{0.85}} & \revise{0.81} & \revise{0.82} \\
    \bottomrule
    \end{tabular}
\end{table*}

\subsection{Accuracy and Robustness Results}
\label{sec: acc-rob}
\revise{In our first experiment} we assess whether large-spread ensembles are effective at classification and we analyze their robustness properties. Indeed, the large-spread condition enforced on the ensemble limits the model shape, thus potentially reducing its predictive power with respect to traditional tree ensembles. Since we are not just concerned about accuracy but we target robustness, we also analyze how large-spread ensembles fare against evasion attacks.
\revise{Our evaluation consists of two parts. We first compare the accuracy and robustness of the large-spread ensembles against traditional random forests of the same size, considering an $L_{\infty}$-attacker. The robustness of the traditional models is computed using SILVA, since CARVE can only be used for verifying large-spread ensembles. We set a timeout per instance of one second, as in~\cite{RanzatoZ20}. Then, we use CARVE to verify the robustness of large-spread ensembles against $L_{1}$-attackers and $L_{2}$-attackers that are not supported by SILVA.}

\subsubsection{Comparison for $L_{\infty}$-norm Attackers.}
\writtenbyLC{Table~\ref{tab:measures} shows the experimental results of our \revise{comparison}. Note that the value of robustness may be approximated, since SILVA may not be able to verify robustness on some instances within the time limit; for these cases, we provide lower and upper bounds of robustness, using the $\pm$ notation. The results highlight that the large-spread ensembles are \textit{reasonably accurate} and often \textit{more robust} than the random forests of the same size. In particular, the accuracy of the large-spread ensembles is at most $0.03$ lower than the accuracy of the corresponding traditional model in the majority of the cases, while the improvement in robustness is at least \revise{$0.04$} in around half of the cases. This is reassuring, because accuracy was at stake, since the large-spread condition restricts the shape of the ensemble and might be associated to a reduction of predictive power. The increase of robustness is an interesting byproduct of the large-spread condition: since thresholds in different trees are far way, evasion attacks are empirically harder to craft. Observe that the accuracy and robustness values of the large-spread ensembles on the MNIST and Fashion-MNIST test sets show that large-spread models present better performance overall than the traditional ensembles. The accuracy of the large-spread ensembles on these two test sets is usually equal to the one of the traditional ensembles, while the robustness value improves of at least \revise{$0.06$} in half of the cases, in particular when the largest considered perturbation $k$ is used as the attacker's capability. For example, the robustness of the large-spread ensemble with $101$ trees of maximum depth $6$ and perturbation $0.0150$ is at least \revise{$0.22$} higher than the robustness of the corresponding random forest, while the accuracy decreases only by \revise{$0.04$} at most. When the value of the perturbation $k$ is the lowest considered, the results are still positive, since the large-spread ensembles \revise{present} the same accuracy and \revise{a higher robustness than the ones of the traditional ensembles.}}

We see a slightly different trend in the results for the REWEMA \revise{and Webspam datasets}: the robustness of large-spread ensembles is always equal to or greater than the robustness of the traditional ensembles, but the gap in accuracy with respect to the traditional ensembles may increase\revise{, in particular when considering large adversarial perturbations, which make it harder to enforce the large-spread condition}. For example, the large spread ensembles of $101$ trees with maximum depth $6$ \revise{trained on the two datasets present $0.88$ and $0.82$ robustness with perturbation $0.015$ and $0.0006$ (respectively, $+0.10$ and $+0.01$ than the robustness of the corresponding traditional tree ensembles), but their accuracy is $0.88$ and $0.85$ (respectively, $-0.10$ and $-0.09$ than the accuracy of the traditional tree ensembles).} This confirms that an improvement in robustness often occurs at the price of a decrease in accuracy, because of the classic trade-off between accuracy and robustness~\cite{TsiprasSETM19, MullerE0V23}. Even in these cases though, adopting large-spread ensembles continues to be useful: the accuracy is always way above the majority class distribution, so the model is usable in the non-adversarial setting, while being normally more robust than the traditional counterpart and amenable for efficient security verification. 
\revise{To explain the observed drop in accuracy for large-spread models, we compare the \emph{permutation feature importance}~\cite{Breiman01} for traditional ensembles and large-spread ensembles to assess which features have more predictive power according to the different models. The analysis is quite interesting. For REWEMA, it shows that traditional models give significant importance to a few numerical features which are less important for large-spread models; large-spread models, in turn, privilege some categorical / ordinal features which are less important for traditional models. Instead, for Webspam, it shows that traditional and large-spread models privilege numerical features with many distinct values. However, the traditional models give also importance to some features with a very skewed empirical distribution towards the value 0, while the large-spread ensembles give more importance to features with scattered values. This motivates why large-spread models sacrifice some predictive power, but show better robustness in general: categorical / ordinal features and, in general, features with more scattered values are harder to target for $L_p$-norm attackers, because their sparse nature makes them more robust to adversarial perturbations, i.e., larger perturbations are required to actually traverse thresholds and thus affect predictions.}

\subsubsection{Additional Attackers}
\revise{Table~\ref{tab:robustness-measures} shows the robustness of the trained large-spread ensembles against different $L_p$-attackers for $p \in \{1,2,\infty\}$. As expected, the large-spread ensembles trained on MNIST and Fashion-MNIST are generally more robust against the weakest $L_1$-attacker and less robust against the strongest $L_{\infty}$-attacker. Instead, we observe that the large-spread ensembles trained on the REWEMA and Webspam datasets show a different behaviour: the robustness values of the large-spread ensemble models are almost the same for every attacker considered. This is explained by the fact that large-spread models trained over such datasets make a more significant use of categorical / ordinal features and features with more scattered values, as discussed in the previous section. The attacker thus cannot perturb the test instances to cross thresholds of important features for prediction, independently of the chosen $L_p$-norm. We remark here that the effectiveness of CARVE does not depend upon $p$: robustness verification is always exact and the complexity of the analysis is independent from $p$. This motivates why the rest of our evaluation only considers the case $p = \infty$.}

\begin{table}[t]
\caption{Robustness measures for large-spread ensembles against different $L_p$-attackers.}
\label{tab:robustness-measures}
    \centering
    \begin{tabular}{c|c|c|c|c|c|c}
    \toprule
    \multirow{2}{*}{\textbf{Dataset}} & \multirow{2}{*}{\textbf{$k$}} & \multirow{2}{*}{\textbf{Trees}} & \multirow{2}{*}{\textbf{Depth}} & \multicolumn{3}{c}{\textbf{Robustness}} \\
    \cmidrule{5-7} 
    & & & & $A_{\infty,k}$ & $A_{2,k}$ & $A_{1,k}$\\
    \midrule
    \multirow{6}{*}{\revise{Fashion-MNIST}} & \multirow{2}{*}{\revise{$0.0050$}} & \revise{25} & \revise{4} & \revise{0.90} & \revise{0.90} & \revise{0.90}\\
    & & \revise{101} & \revise{6} & \revise{0.93} & \revise{0.93} & \revise{0.94} \\
    \cline{2-7}
    & \multirow{2}{*}{\revise{$0.0100$}} & \revise{25} & \revise{4} & \revise{0.87} & \revise{0.88} & \revise{0.89} \\
    & & \revise{101} & \revise{6} & \revise{0.91} & \revise{0.91} & \revise{0.93} \\
    \cline{2-7}
    & \multirow{2}{*}{\revise{$0.0150$}} & \revise{25} & \revise{4} & \revise{0.88} & \revise{0.89} & \revise{0.89} \\
    & & \revise{101} & \revise{6} & \revise{0.89} & \revise{0.89} & \revise{0.91} \\
    \midrule
    \multirow{6}{*}{\revise{MNIST}} & \multirow{2}{*}{\revise{$0.0050$}} & \revise{25} & \revise{4} & \revise{0.96} & \revise{0.96} & \revise{0.97} \\
    & & \revise{101} & \revise{6} & \revise{0.97} & \revise{0.98} & \revise{0.98} \\
    \cline{2-7}
    & \multirow{2}{*}{\revise{$0.0100$}} & \revise{25} & \revise{4} & \revise{0.90} & \revise{0.93} & \revise{0.95} \\
    & & \revise{101} & \revise{6} & \revise{0.97} & \revise{0.98} & \revise{0.98} \\
    \cline{2-7}
    & \multirow{2}{*}{\revise{$0.0150$}} & \revise{25} & \revise{4} & \revise{0.83} & \revise{0.88} & \revise{0.93} \\
    & & \revise{101} & \revise{6} & \revise{0.94} & \revise{0.95} & \revise{0.97} \\
    \midrule
    \multirow{6}{*}{\revise{REWEMA}} & \multirow{2}{*}{\revise{$0.0050$}} & \revise{25} & \revise{4} & \revise{0.87} & \revise{0.87} & \revise{0.87} \\
    & & \revise{101} & \revise{6} & \revise{0.89} & \revise{0.89} & \revise{0.89}\\
    \cline{2-7}
    & \multirow{2}{*}{\revise{$0.0100$}} & \revise{25} & \revise{4} & \revise{0.87} & \revise{0.87} & \revise{0.87}\\
    & & \revise{101} & \revise{6} & \revise{0.88} & \revise{0.88} & \revise{0.88}\\
    \cline{2-7}
    & \multirow{2}{*}{\revise{$0.0150$}} & \revise{25} & \revise{4} & \revise{0.85} & \revise{0.87} & \revise{0.87}\\
    & & \revise{101} & \revise{6} & \revise{0.88} & \revise{0.88} & \revise{0.88}\\
    \midrule
    \multirow{6}{*}{\revise{Webspam}} & \multirow{2}{*}{\revise{$0.0002$}} & \revise{25} & \revise{4} & \revise{0.87} & \revise{0.88} & \revise{0.88}\\
    & & \revise{101} & \revise{6} & \revise{0.90} & \revise{0.90} & \revise{0.90}\\
    \cline{2-7}
    & \multirow{2}{*}{\revise{$0.0004$}} & \revise{25} & \revise{4} & \revise{0.86} & \revise{0.86} & \revise{0.86}\\
    & & \revise{101} & \revise{6} & \revise{0.86} & \revise{0.86} & \revise{0.87}\\
    \cline{2-7}
    & \multirow{2}{*}{\revise{$0.0006$}} & \revise{25} & \revise{4} & \revise{0.85} & \revise{0.86} & \revise{0.86}\\
    & & \revise{101} & \revise{6} & \revise{0.82} & \revise{0.83} & \revise{0.83}\\
    \bottomrule
    \end{tabular}
\end{table}

\subsection{Efficiency of Robustness Verification}
\writtenbyLC{We now compare the SILVA and CARVE robustness verification tools along two different dimensions: verification time and memory consumption. For simplicity, we only focus on the verification of large ensembles with $101$ trees and maximum depth $6$ on the MNIST dataset with $k=0.0150$. As emerged from the results in Section~\ref{sec: acc-rob}, this is a setting where a state-of-the-art approach like SILVA clearly shows its limits: indeed, SILVA could not provide a precise estimate of the robustness of this model ($\pm$ 0.05). In order to measure the verification time per instance and setting timeouts in the same way for both the tools, we use the GNU commands \texttt{time} and \texttt{timeout} that measure the elapsed wall clock time. The former command is also used to compute the maximum amount of physical memory allocated to the verifier. When it is required to set a maximum amount of physical memory that the process can use, we use the Linux kernel feature \texttt{cgroup}. All the experiments are performed on a virtual machine with 103 GB of RAM and Ubuntu 20.04.4 LTS, running on a server with an Intel Xeon Gold 6148 2.40GHz.}

\begin{figure*}[t]
  \centering
    \subfloat[Number of verified instances of the test set when varying the time limit in seconds for the verification.\label{fig:time-efficiency}]{\includegraphics[width=.42\textwidth]{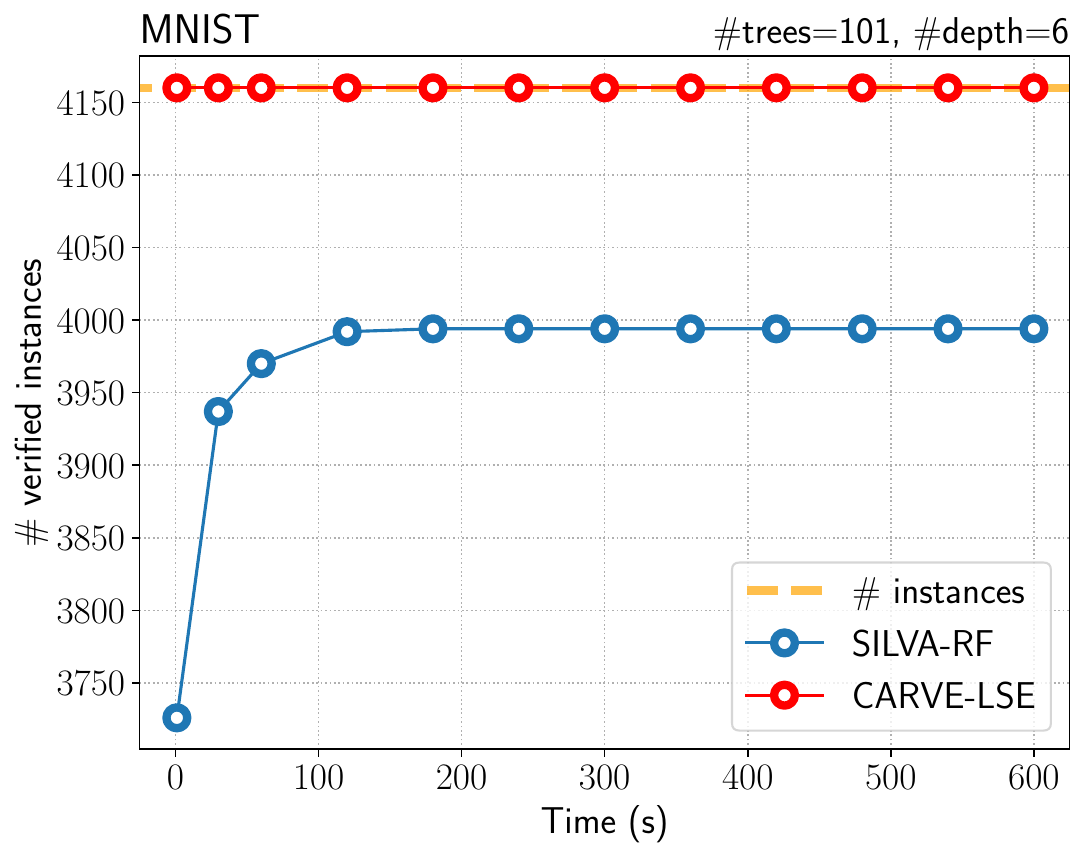}}
  \hfill
    \subfloat[Number of verified instances of the test set when varying the maximum memory consumption limit for the verification.\label{fig:memory-efficiency}]{\includegraphics[width=.42\textwidth]{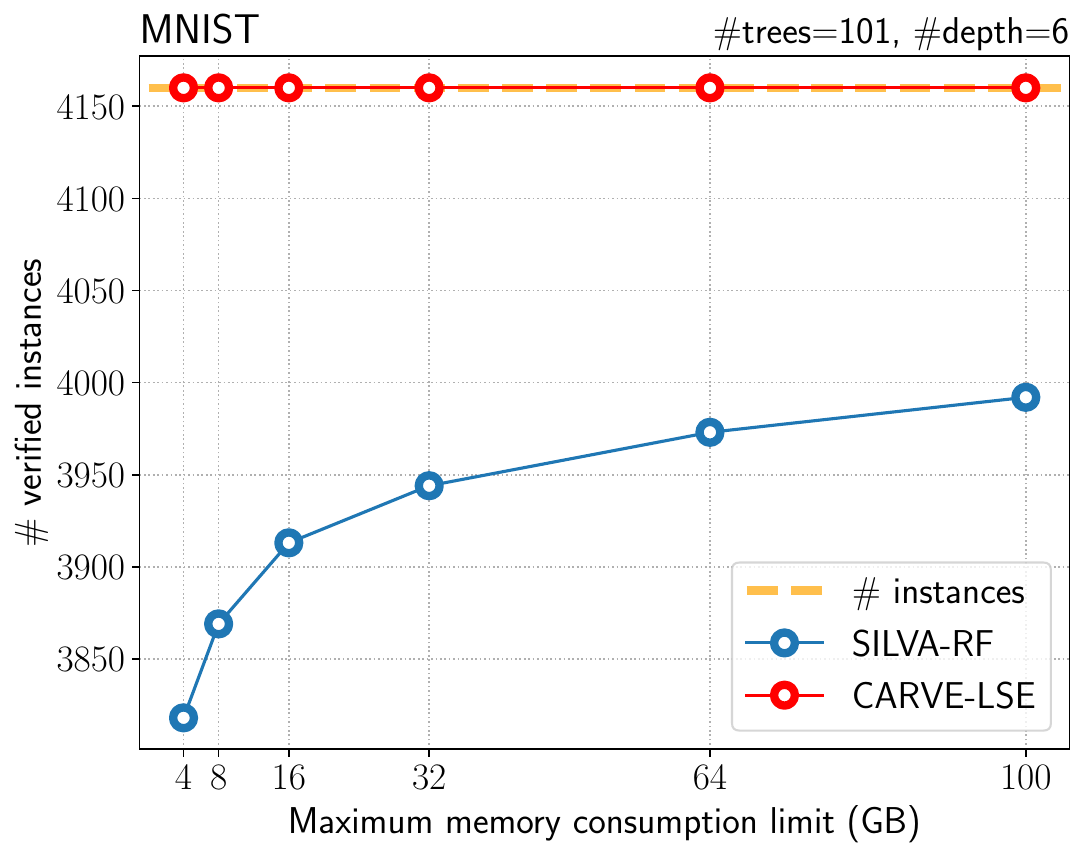}}
  \caption{Comparison of the time and memory efficiency of SILVA and CARVE on the MNIST dataset (we consider ensembles with 101 trees of maximum depth 6).}
  \label{fig: efficiency}
\end{figure*}

\subsubsection{Time Efficiency}
In our first experiment we compare the robustness verification times for traditional tree ensembles using SILVA and the robustness verification times for large-spread ensembles using CARVE. This way, we compare a state-of-the-approach for adversarial machine learning models (i.e., what we would do today) against our custom algorithm designed to take advantage of the large-spread condition (i.e., what we put forward in this paper). \writtenbyLC{In the experiments of Section~\ref{sec: acc-rob}, we set the maximum verification time per instance of SILVA to one second. However, SILVA may complete the verification also on more difficult instances if more time is granted, e.g., 60 seconds~\cite{RanzatoZ20}. In order to perform a fair comparison, we compare how many instances of the MNIST test set can be verified under growing time limits per instance, i.e., from one second to 10 minutes. This methodology allows us to figure out on how many instances the verification is really difficult. Note that the timeout of 10 minutes per instance is already extremely large, since test sets normally include thousands of instances.}

\writtenbyLC{Figure~\ref{fig:time-efficiency} shows the results of our experiment. The plot shows that SILVA is not able to verify the robustness of the traditional tree ensemble on 434 instances in less than one second and on 190 instances in less than one minute, providing just approximate robustness estimates with an uncertainty of 0.10 ($\pm 0.05$) and 0.05 ($\pm 0.025$) respectively. On the other hand, our tool CARVE requires \emph{less than one second} per instance to verify the robustness of the large-spread ensemble on all the instances of the test set, providing an exact estimate of the robustness of the model. As the maximum amount of verification time per instance increases, the number of instances on which SILVA is not able to verify the robustness of the model further decreases, e.g., 168 instances with a timeout of 120 seconds and 166 instances with a timeout of 180 seconds. Even though the robustness estimate of SILVA becomes more precise as the timeout per instance increases, i.e., the uncertainty on robustness decreases to 0.04 ($\pm 0.02$) with a timeout per instance of 180 seconds, this process eventually hits a wall: the remaining 166 instances cannot be verified even when the timeout increases to 10 minutes per instance. Moreover, the improved precision comes at the cost of an higher total verification time: with a timeout of 120 seconds, SILVA requires in total 22,220 seconds to verify the traditional tree ensemble on the entire MNIST test set, while CARVE requires just 129 seconds in total, i.e., a reduction of two orders of magnitude. As expected, the results show the pitfalls of the complete robustness verification on traditional tree-ensembles and the improvements in the verification time enabled by the large-spread condition. Since the verification problem is NP-complete, there may be instances on which the verification time increases exponentially, while the large-spread condition allows one to train tree ensembles whose robustness can always be verified in polynomial time.}



\subsubsection{Memory Efficiency}
\writtenbyLC{Our first experiment provides only a partial picture of the efficiency of the robustness verification and the reasons for the potential inefficiency of SILVA. Indeed, memory constraints should also be taken into account during robustness verification, since a high memory consumption may make the verification unfeasible on standard commercial systems.}

\writtenbyLC{In our second experiment, we compare the memory efficiency of SILVA and CARVE. In particular, we compare how many instances can be verified given a growing maximum memory consumption limit per instance, setting the maximum amount of verification time per instance to 10 minutes. The results of our experiment are shown in Figure~\ref{fig:memory-efficiency}. The results highlight that SILVA may consume a lot of memory in order to provide precise robustness estimates. In the best scenario, with 100 GB of memory available, SILVA is still unable to verify the robustness of the model on 168 instances, providing just an approximate estimate of the robustness of the traditional tree-based ensemble with an uncertainty of 0.04 ($\pm 0.02$). Even though the interval on which the robustness approximation is not so large in this setting, the plot shows that the number of instances that SILVA can not verify increases as the memory consumption limit decreases, expanding also the uncertainty of SILVA in the robustness estimation. For example, SILVA is not able to verify the robustness of the model on 216 and 342 instances with the memory consumption limit of 32 GB and 4 GB respectively, providing an uncertainty in the robustness estimates of 0.05 ($\pm 0.025$) and 0.08 ($\pm 0.04$). Instead, CARVE manages to verify the robustness of the large-spread ensemble on all the MNIST test set using \emph{less than 4GB of memory} per instance, providing an exact value of robustness. More precisely, the maximum memory consumption by CARVE is less than 1 GB in practice. The results confirm the efficiency in terms of memory consumption of our proposal and the unfeasibility of obtaining an exact value of robustness on traditional tree ensembles using a state-of-the-art verifier like SILVA when memory consumption constraints are imposed.}



\subsubsection{Efficiency Under Time and Memory Constraints}
\writtenbyLC{We finally perform a comparison between CARVE and SILVA when enforcing both a maximum verification time limit and a maximum memory consumption limit. In particular, we compare the total verification time, the maximum memory consumption and the number of instances on which the tool is not able to return an answer given a maximum verification time of 60 seconds per instance and a maximum memory consumption limit of 64 GB.} 

\begin{table}
\caption{Comparison of total verification time and maximum memory consumption of SILVA and CARVE on the MNIST test set. The last column reports the number of instances on which the verifier was not able to provide an answer because it exceeded the time or memory limits.}
\label{tab:efficiency-two-constraints}
    \centering
    \begin{tabular}{c|c|c|c}
    \toprule
    \textbf{Tool} & \textbf{Total Time (s)} & \textbf{Memory (GB)} & \textbf{\# Failures}\\
    \midrule
    SILVA & 14,448 & 64 & 190 \\
    CARVE & 129 & 0.03 & 0 \\
    \bottomrule
    \end{tabular}
\end{table}

    
Table~\ref{tab:efficiency-two-constraints} shows the results of our experiment. The results confirm the observations from the previous sections: CARVE is far more efficient of SILVA in terms of 
both verification time and memory consumption. In particular, CARVE outperforms SILVA on the total verification time on the MNIST test set, verifying the large-spread ensemble on all the instances in just 129 seconds, thus being 112 times faster than SILVA (that requires 14,448 seconds). Moreover, the memory consumption of CARVE is more than 2,000 times lower than the one of SILVA, using just 0.03 GB of memory capacity, thus CARVE is usable on commodity hardware. Finally, SILVA is not able to provide an answer on 190 instances of the test set, providing an approximated robustness estimate with an uncertainty of 0.05 ($\pm 0.025$), while CARVE is able to provide the exact robustness value. \revise{This provides clear evidence of the challenges of robustness verification for traditional tree ensembles: since robustness verification is NP-hard in general, even a state-of-the-art tool like SILVA is bound to fail on specific inputs.}

\subsection{Efficiency of the Training Algorithm}
\label{sec: performance}
\revise{Finally, we evaluate the time efficiency of the training algorithm for large-spread ensembles (Algorithm~\ref{alg:training}). Intuitively, the difficulty of enforcing the large-spread conditions depends on two factors: the model size and the adversarial perturbation $k$. Indeed, the larger is $k$, the higher becomes the distance to be enforced across thresholds in different trees. We then perform two experiments, each for different values of $k$: in the first, we fix the maximum tree depth at six and we vary the number of trees in $\{25, 51, 75, 101\}$; in the second, we fix the number of trees at $101$ and we vary the maximum depth of the trees in $\{3, 4, 5, 6\}$. The presented times are measured for a specific hyper-parameter choice enabling successful training in all settings ($MAX\_ITER = 500$, $MULT = 6$, $INTV = [k,1.5k]$, $l = 6$).}

\subsubsection{Number of Trees}
\revise{Figure~\ref{fig:time-trees} shows the results of our first experiment. We observe that the time required for training a large-spread ensemble depends on the dataset, most likely because enforcing the large-spread condition might be easier or harder for different training data. When considering a number of trees less than or equal to $75$, the time required for training a large-spread ensemble is less than 150 seconds for all the considered datasets and adversarial perturbations. For example, the time required for training a large-spread ensemble of $75$ trees is 28 seconds on MNIST and 145 seconds on Webspam when considering the largest adversarial perturbation. Similarly, training a large-spread ensemble is efficient when considering smaller adversarial perturbations: for the smallest perturbations, training time ranges from one second on the REWEMA dataset to 16 seconds on the Webspam dataset. This result is encouraging, because the trained models already obtain a reasonable accuracy on the test set and the range of adversarial perturbations might be small in practical cases.}

\revise{On the downside, when considering larger models with 101 trees, the role of the adversarial perturbations on the training time becomes more significant. For example, training a large-spread ensemble with $101$ trees under the largest adversarial perturbations required 137 seconds on MNIST and 1,835 seconds on Webspam. The motivation is that the cost of adding a tree to the ensemble increases as the size of the ensemble increases, because all the thresholds of the current ensemble must be adjusted with respect to the new tree. Fixing such violations to the large-spread condition is difficult for larger adversarial perturbations, because thresholds must be pushed farther away. This fact particularly affects the time required for training large-spread ensembles on the Webspam dataset: since some important features for the ensemble have a very skewed empirical distribution, the thresholds learned by the traditional tree-based ensembles for these features are close, thus separating them in an effective way is difficult and may require the training algorithm to perform many iterations.}

\subsubsection{Maximum Tree Depth}
\revise{Figure~\ref{fig:time-depth} shows the results of our second experiment. We observe that training a large-spread ensemble of depth at most five requires at most 122 seconds for all the considered datasets and adversarial perturbations. For example, training a large-spread ensemble of $101$ trees with maximum depth five takes 55 seconds on the Fashion-MNIST dataset and 122 seconds on the Webspam dataset. Moreover, the results confirm that training a large-spread ensemble considering small adversarial perturbations is efficient, e.g., the maximum time required for training a large-spread ensemble of $101$ trees with maximum depth six, considering the smallest adversarial perturbation for each dataset, is 35 seconds.}

\revise{However, we observe that, when considering large-spread ensembles with deeper trees, choosing a higher adversarial perturbation may determine a considerable increase in the time required for the training. The worst case is observed on the Webspam dataset, where the time required for training a large-spread ensemble of $101$ trees with maximum depth six and $k=0.0006$ is 1,835 seconds. Indeed, increasing the value of the depth of the trees in the ensemble causes an exponential growth in the number of nodes of the ensemble and enforcing the large-spread condition for higher perturbations is more difficult, thus more violations of the large-spread condition need to be fixed to add a single tree to the ensemble.}

\subsubsection{Discussion}
\revise{Our experimental evaluation shows that the training algorithm for large-spread ensembles is efficient when the model size is relatively limited ($\leq 75$ trees) or the adversarial perturbation is small. Concretely, the most challenging model including 75 trees could be trained in 145 seconds, while the most challenging model for the smallest adversarial perturbation could be trained in 35 seconds. When combining large model size with large adversarial perturbations, however, the training time can become higher. The worst case was observed on the Webspam dataset, where a model with 101 trees required 1,835 seconds to be trained under the largest adversarial perturbation. Nevertheless, this price is just paid for training: once the model is trained, robustness can be verified in polynomial time for thousands of instances. Also, such extreme cases only occurred on the Webspam dataset: for example, the most challenging models to train on Fashion-MNIST and REWEMA took just 113 seconds and 13 seconds respectively. We find these results appropriate for our first evaluation of large-spread ensembles, in particular because our implementation of LSE is not heavily optimized, and we plan to design more efficient training algorithms for large-spread ensembles as future work.}

\begin{figure*}[t]
    \centering
    \subfloat[]{\includegraphics[width=.24\textwidth]{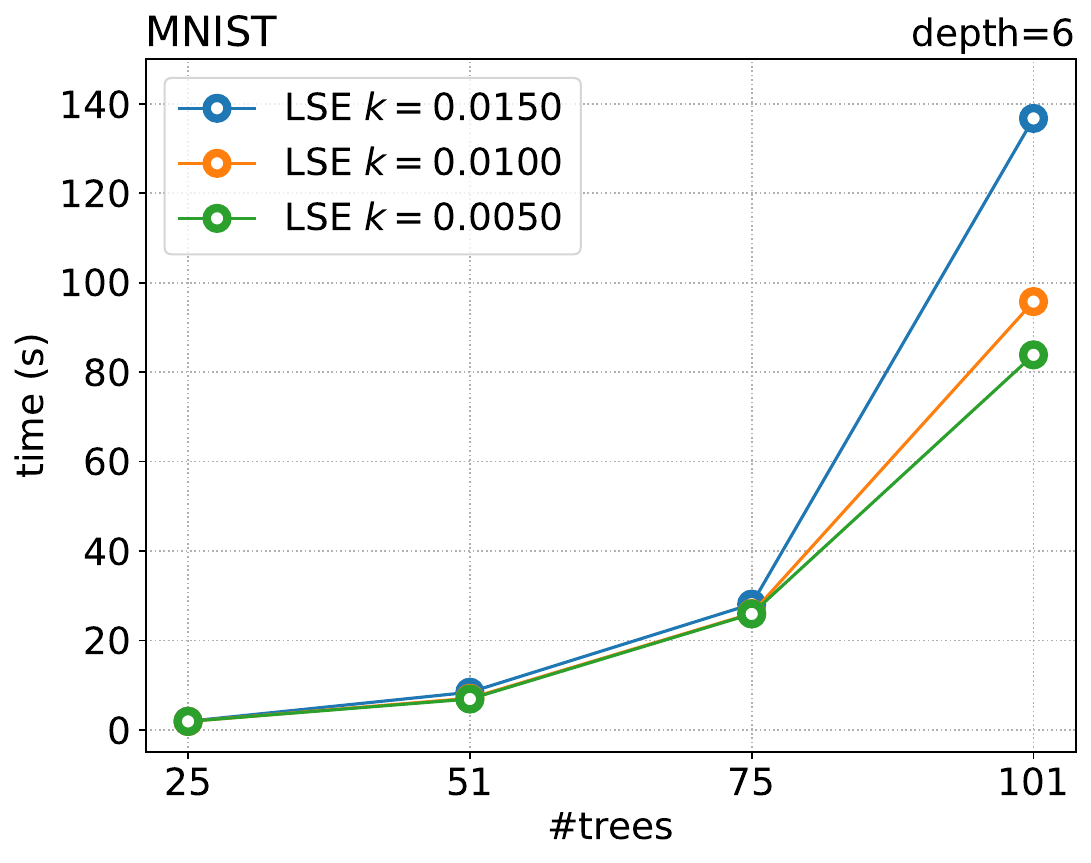}}
  \hfill
    \subfloat[]{\includegraphics[width=.24\textwidth]{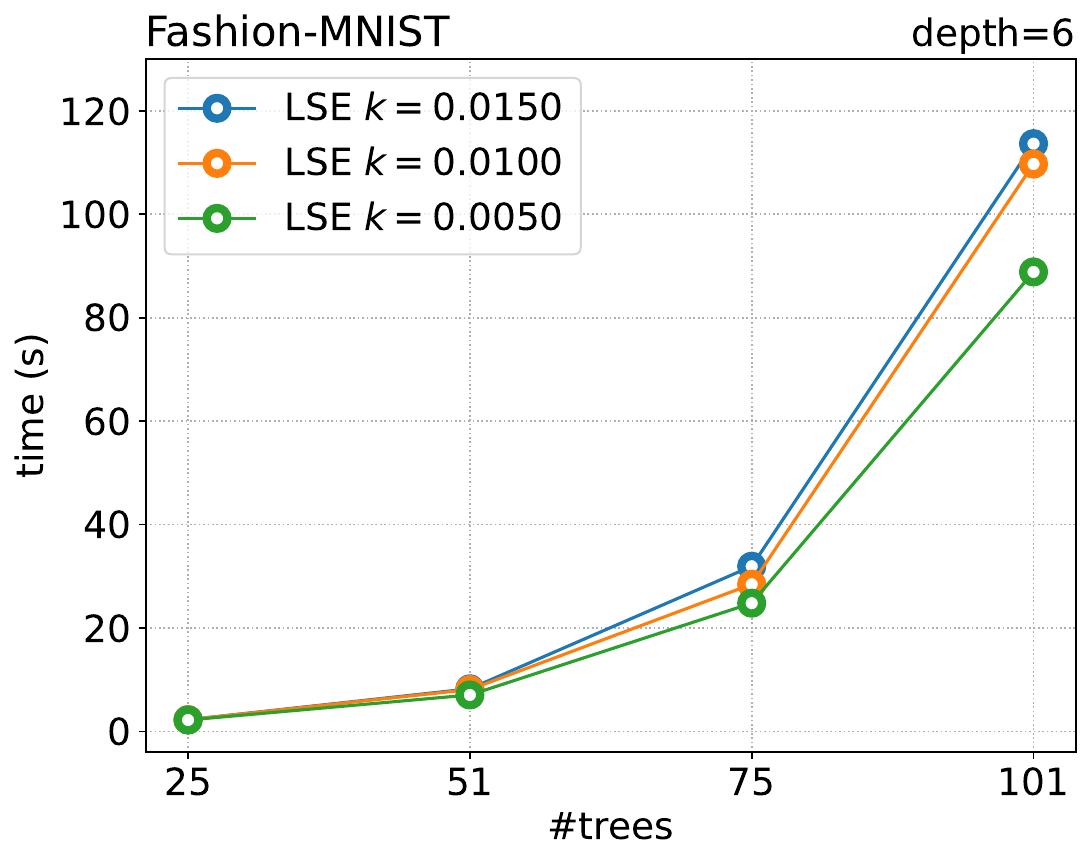}}
  \hfill
    \subfloat[]{\includegraphics[width=.24\textwidth]{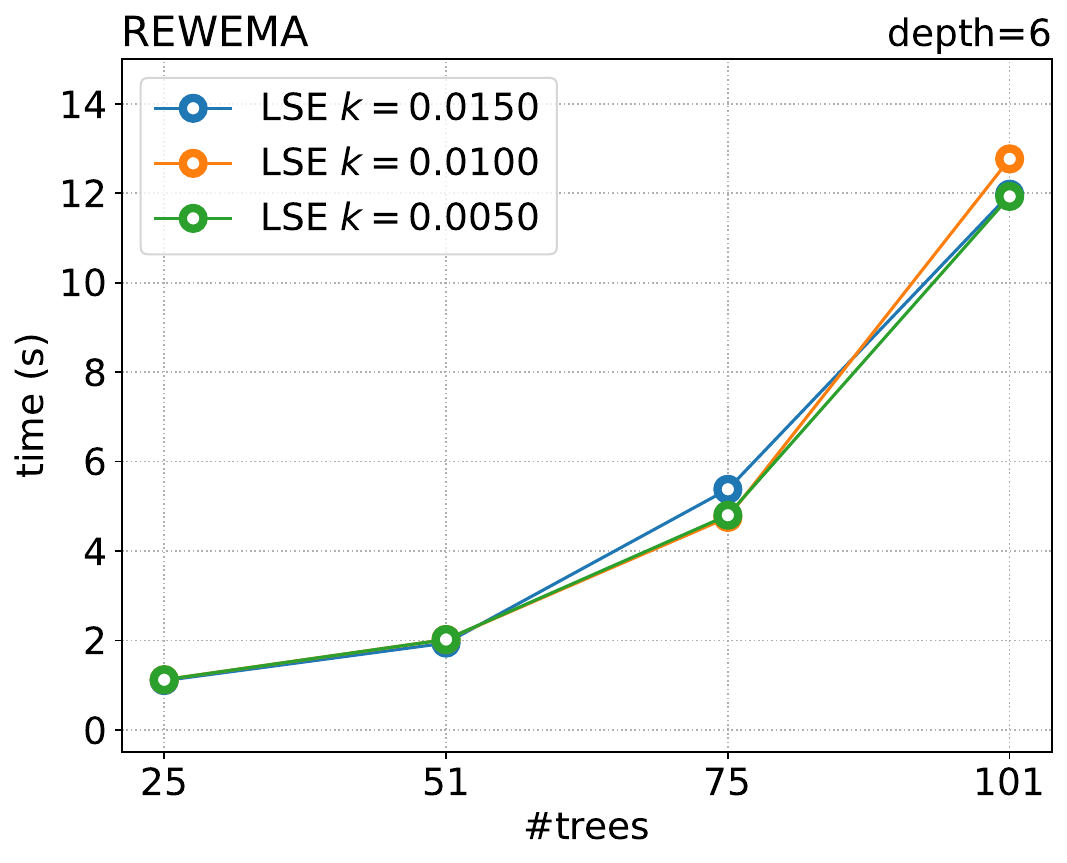}}
  \hfill
    \subfloat[]{\includegraphics[width=.24\textwidth]{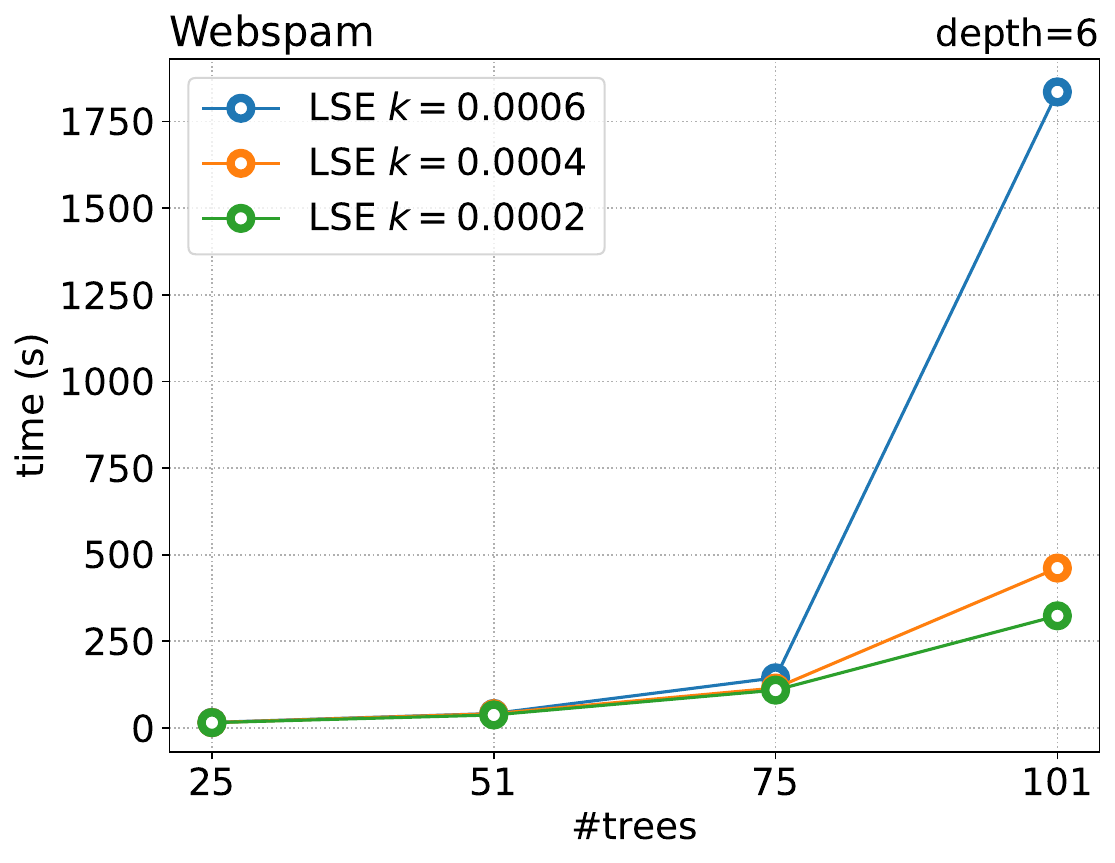}}
  \caption{Efficiency of LSE when varying the number of trees of the large-spread ensemble.}
  \label{fig:time-trees}
\end{figure*}

\begin{figure*}[t]
    \centering
    \subfloat[]{\includegraphics[width=.24\textwidth]{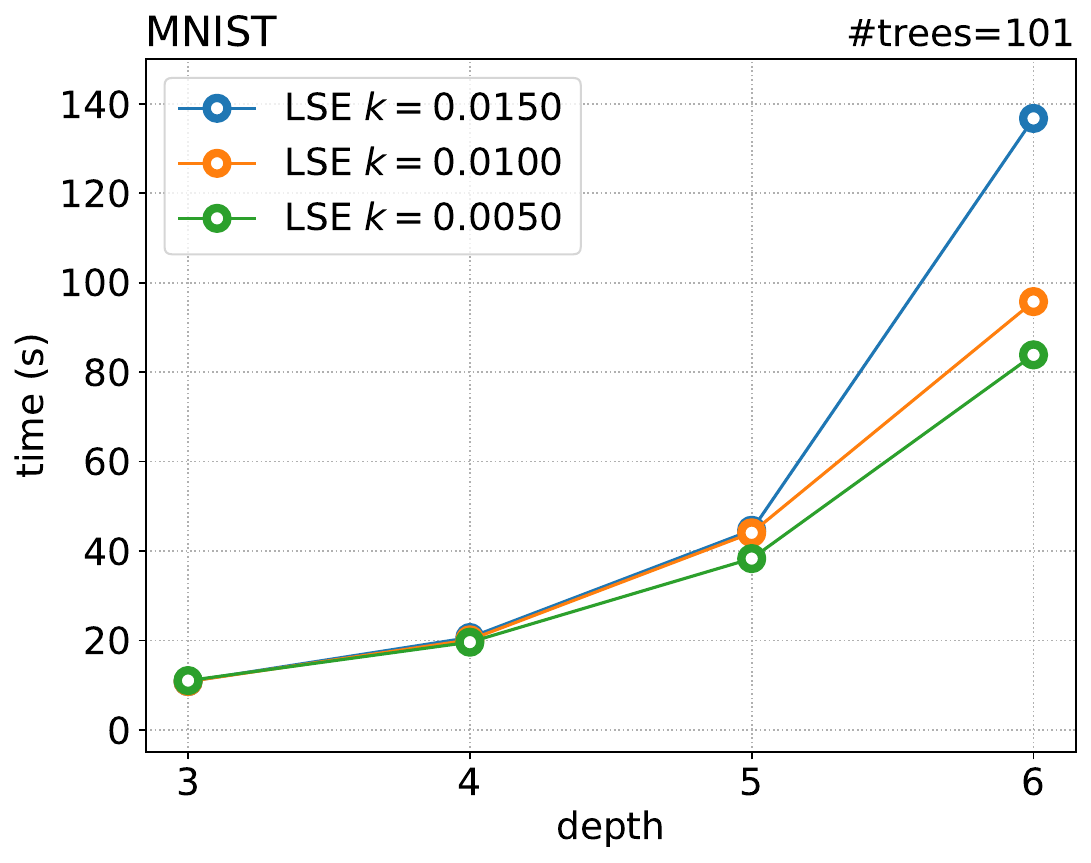}}
  \hfill
    \subfloat[]{\includegraphics[width=.24\textwidth]{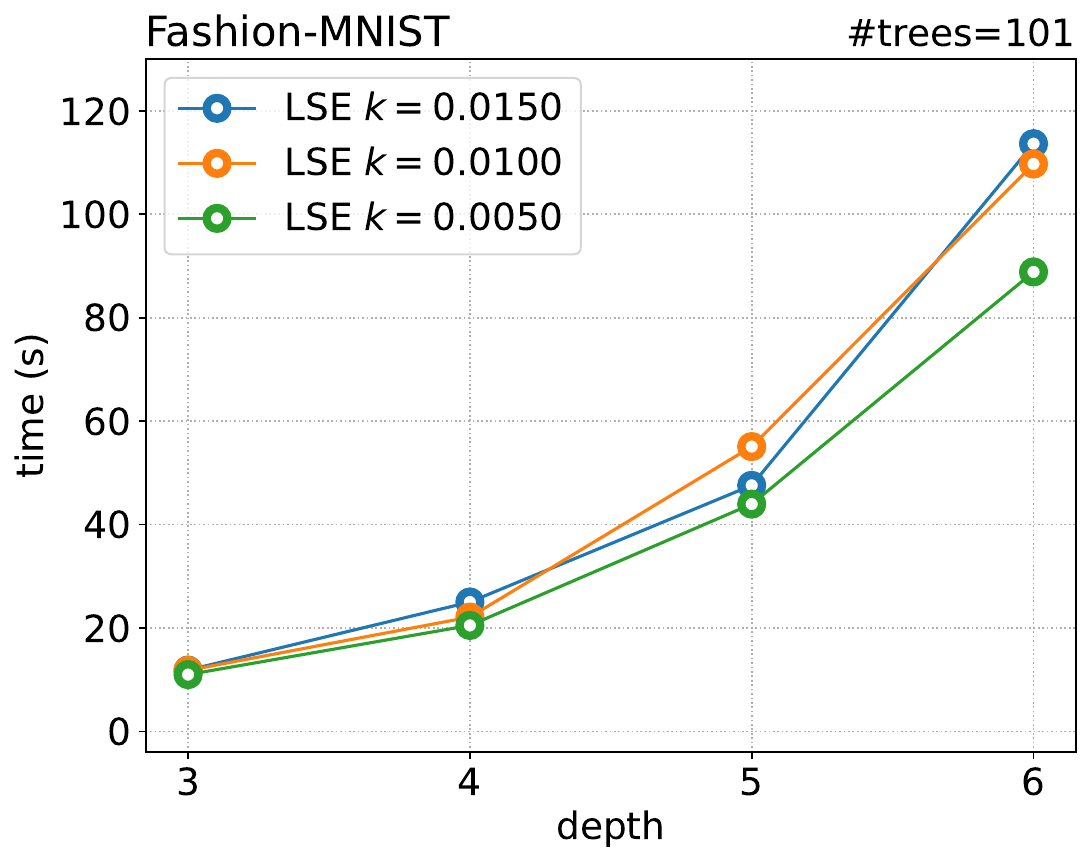}}
  \hfill
    \subfloat[]{\includegraphics[width=.24\textwidth]{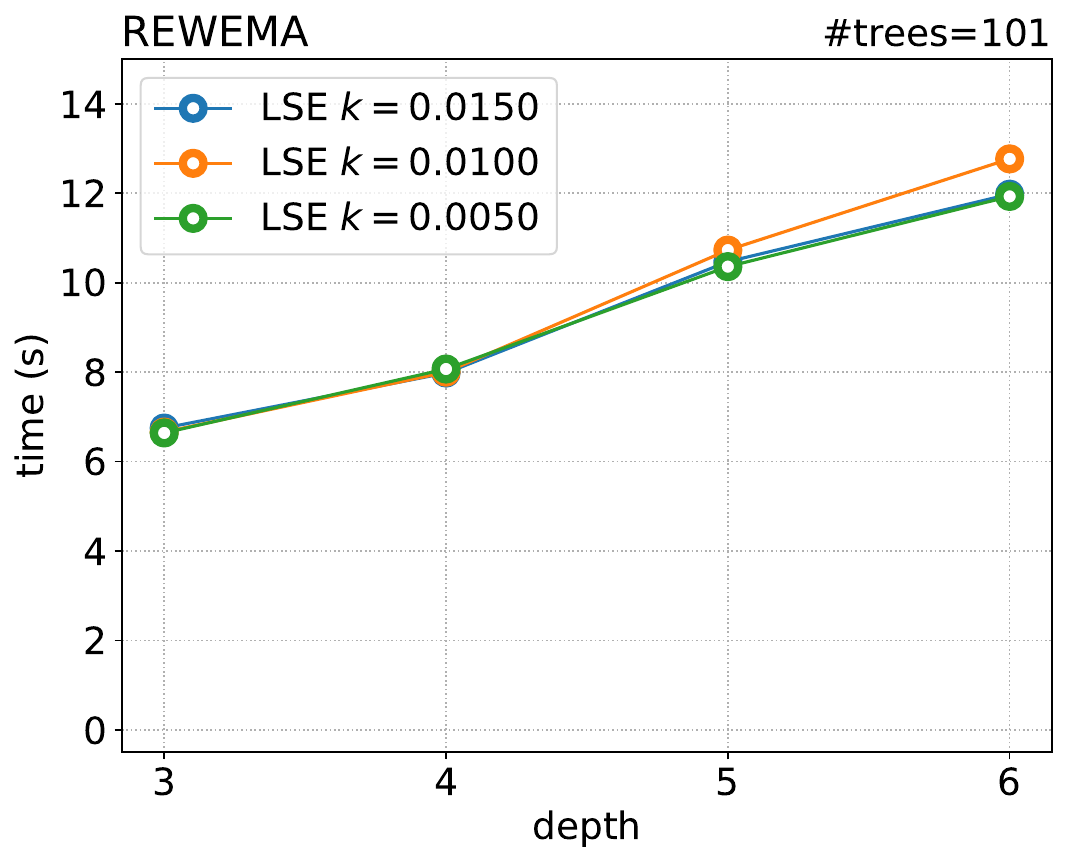}}
  \hfill
    \subfloat[]{\includegraphics[width=.24\textwidth]{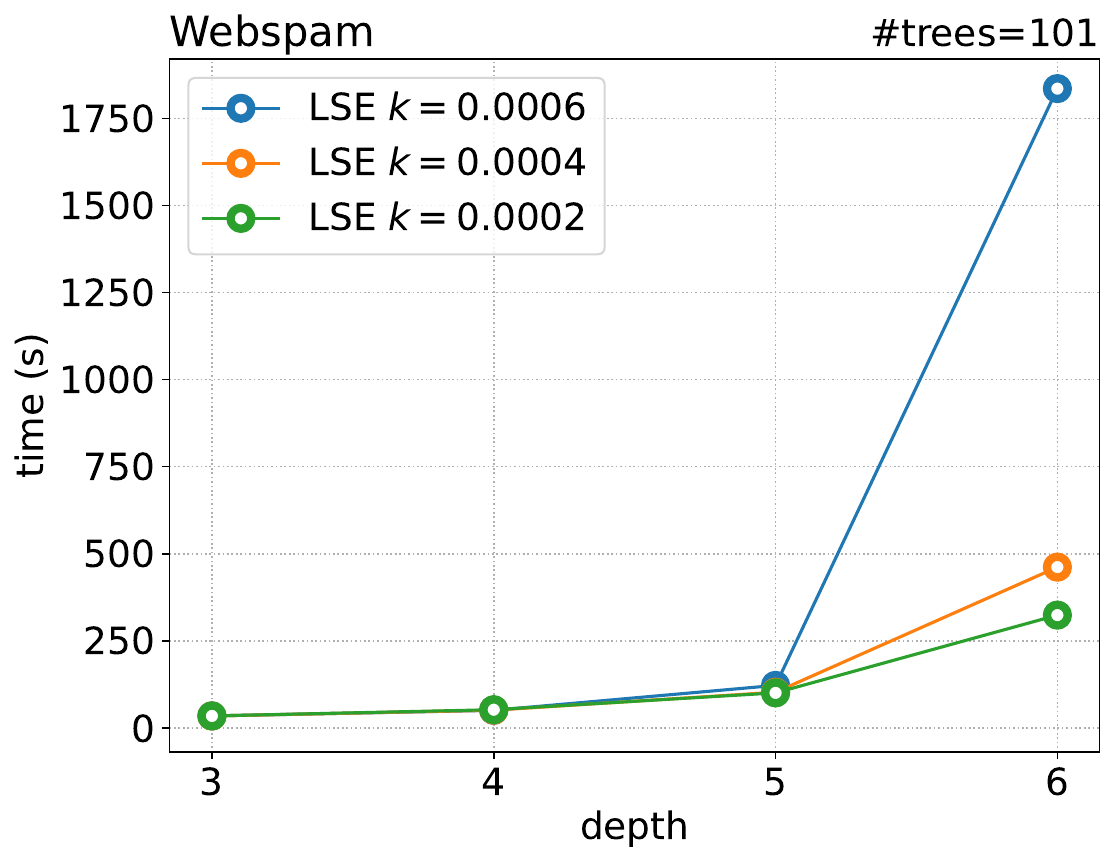}}
  \caption{Efficiency of LSE when varying the maximum depth of the trees of the large-spread ensemble.}
  \label{fig:time-depth}
\end{figure*}

\subsection{Take-Away Messages}
Our experimental evaluation shows that:
\begin{itemize}
    \item Large-spread ensembles sacrifice some predictive power with respect to traditional tree ensembles, yet their accuracy remains way higher than the majority class of the test set. Even better, in several cases the accuracy of large-spread ensembles is equal to the accuracy of traditional tree ensembles.

    \item Large-spread ensembles are generally more robust than traditional tree ensembles. This empirical observation is a useful byproduct of the large-spread condition, which makes it harder to craft evasion attacks which are effective against multiple trees in the ensemble.

    \item Our verification tool for large-spread ensemble CARVE is much more efficient than SILVA, a state-of-the-art verifier for traditional tree ensembles. Improvements are due to both verification time and memory consumption.

    \item \revise{Our training tool LSE is efficient for training large-spread ensembles of moderate size or when considering small adversarial perturbations. The time required for training a large-spread ensemble may increase when the model is large and the considered perturbation is high for the chosen dataset.}
\end{itemize}

Moreover, we showed that SILVA can provide just approximate robustness estimates in some experimental settings, even when provided with extremely high time and memory bounds (10 minutes per instance, 100 GB of RAM). Conversely, CARVE can compute the exact value of robustness using just limited time and memory (1 second per instance, 1 GB of RAM). This shows the effectiveness of the verifiable learning paradigm: models trained with formal verification in mind can be verified in a matter of seconds even on traditional commercial hardware, contrary to traditional machine learning models which cannot be accurately verified even when extremely powerful servers are available.
\section{Related Work}
We already mentioned that prior work studied the complexity of the robustness verification problem for decision tree ensembles~\cite{KantchelianTJ16,WangZCBH20}. This problem was proved to be NP-complete for arbitrary $L_p$-norm attackers, even when restricting the model shape to decision stump ensembles~\cite{Andriushchenko019,WangZCBH20}. To the best of our knowledge, we are the first to identify a specific class of decision tree ensembles enabling robustness verification in polynomial time. Prior work on robustness verification for decision tree ensembles proposed different techniques, such as exploiting equivalence classes extracted from the tree ensemble~\cite{TornblomN20}, integer linear programming~\cite{KantchelianTJ16}, a reduction to the max clique problem~\cite{ChenZS0BH19}, abstract interpretation~\cite{RanzatoZ20, CalzavaraFL20} and satisfiability modulo theory (SMT) solving~\cite{DevosMD21, EinzigerGSS19, SatoKNO20}. Though effective in many cases, these techniques still have to deal with the exponential complexity of the robustness verification problem, so they are bound to fail for large ensembles and complex datasets. We experimentally showed that a state-of-the-art verifier like SILVA~\cite{RanzatoZ20} is much less efficient than our verifiable learning approach, supporting verification in polynomial time, and can only compute approximate robustness estimates in practical cases. Moreover, our LSE training algorithm produces tree ensembles that are in general more robust than the traditional counterparts as a side-effect of imposing that the thresholds of different trees are sufficiently far away. Several papers in the literature discussed new algorithms for training tree ensembles that are robust to evasion attacks~\cite{CalzavaraLMO21, CalzavaraLTAO20, CalzavaraLT19, ChenZBH19, KantchelianTJ16, Andriushchenko019, VosV21, RanzatoZ21, GuoTGZ22, VosV22-1, VosV22-2, Chen0JCJ21}, but our work is complementary to them. Indeed, our primary goal is not enforcing robustness, which is a byproduct of our training algorithm, but supporting efficient robustness verification of the trained models. \revise{We also acknowledge that our work solely focuses on the classic definition of robustness, known as \emph{local} robustness in more recent literature discussing global robustness and related properties~\cite{Chen0QLJW21,CalzavaraCLMO22,LeinoWF21}. This line of research aims to achieve security verification independently of the choice of a specific test set, enhancing the credibility of security proofs. Given that local robustness remains popular and is easier to deal with, we stick to it in this paper and we leave the extension of our framework to global robustness as future work.}

It is worth mentioning that a lot of work has been done on the robustness verification of deep neural networks (DNNs). Classic approaches for exact verification often do not scale to large DNNs, as for tree ensembles, and they are typically based on SMT ~\cite{KatzBDJK17, KatzHIJLLSTWZDK19, HuangKWW17} and integer linear programming ~\cite{BastaniILVNC16, LomuscioM17, TjengXT19, DuttaJST18}. To mitigate the scalability problems of robustness verification, different proposals have been done, such as shrinking the original DNN through pruning~\cite{GuidottiLPT20} and finding specific classes of DNNs that empirically enable more efficient robustness verification~\cite{JiaR20}. Xiao et. al. ~\cite{XiaoTSM19} proposed the idea of \textit{co-designing} model training and verification, i.e., training models that show reasonable accuracy and robustness, while better enabling exact verification. In particular, their work proposes a training algorithm for DNNs that encourages weight sparsity and ReLU stability, two properties that improve the efficiency of verification through SMT solving. There are significant differences between these lines of work and ours. First, prior techniques only provide empirical efficiency guarantees, while our proposal leads to a formal complexity reduction of the robustness verification problem through the design of a polynomial time algorithm. Moreover, our research deals with tree ensembles rather than DNNs.

Finally, we observe that recent work explored the adversarial robustness of model ensembles~\cite{YangLXK0L22}. The main result of this work proved that the combination of ``diversified gradient'' and ``large confidence margin'' are sufficient and necessary conditions for certifiably robust ensemble models. While this result cannot be directly applied to non-differentiable models such as decision tree ensembles, the intuition of diversifying models is similarly captured by our large-spread condition. We plan to explore any intriguing connections with this proposal as future work.
\section{Conclusion}
We introduced the general idea of \emph{verifiable learning}, i.e., the adoption of training algorithms designed to learn restricted model classes amenable for efficient security verification. We applied this idea to decision tree ensembles, identifying the class of \emph{large-spread} ensembles.
We showed that this class of ensembles
admits robustness verification in polynomial time, whereas the problem is NP-hard for general decision tree models.
We then proposed a pruning-based training algorithm to learn large-spread ensembles from traditional decision tree ensembles.
Our experiments on public datasets show that large-spread ensembles \revise{sacrifice a limited amount} of the predictive power of traditional tree ensembles, but their robustness is normally higher and much more efficient to verify. This makes large-spread ensembles appealing in the adversarial setting.

As future work, we plan to investigate the use of verifiable learning also for other popular model classes, e.g., neural networks. Moreover, we want to explore different training algorithms for large-spread ensembles and compare their effectiveness against the pruning-based approach proposed in this paper.

\paragraph{Acknowledgements.} 
We thank the reviewers for their constructive feedback, which has greatly contributed to the improvement of this paper. This research was supported by project SERICS (PE00000014) under the MUR National Recovery and Resilience Plan funded by the European Union - NextGenerationEU and by project EFRA (101093026) under the Horizon Europe programme.

\bibliographystyle{ACM-Reference-Format}
\bibliography{main}

\iffull

\newpage

\appendix
\section{Proof of Theorem~\ref{thm:ensembles}}
\label{sec:proofs}
We introduce the following auxiliary notation:
$\opt(t,p,k,\vec{x},y) = \argmin_{ \vec{\delta} \in S } ||\vec{\delta}||_p$
where $S=\{ \vec{\delta} \,|\, ||\vec{\delta}||_p \in \Call{Reachable}{t,p,k,\vec{x},y}\}$.
Note that $\Call{Reachable}{t,p,k,\vec{x},y}$ visits every leaf of $t$.
Therefore, $\opt(t,p,k,\vec{x},y)$ is undefined iff
$\Call{Reachable}{t,p,k,\vec{x},y}=\varnothing$, i.e.,
$\Call{Reachable}{}$ does not visit
any leaf with label different from $y$.

\begin{lemma}
\label{lem:trees-reach}
Let $\vec{x}$ be an instance with label $y$ and let $\vec{\delta} = \opt(t,p,k,\vec{x},y)$. Then the following properties hold:
\begin{enumerate}
    \item If $\vec{\delta}$ is undefined, then there does not exist any $\vec{z} \in A_{p,k}(\vec{x})$ such that $t(\vec{z}) \neq y$.
    \item If $\vec{\delta}$ is defined, then $\vec{z}_o = \vec{x} + \vec{\delta} \in A_{p,k}(\vec{x})$ and $t(\vec{z}_o) \neq y$.
    Moreover, for every $\vec{z} \in A_{p,k}(\vec{x})$ such that $t(\vec{z}) \neq y$ we have $||\vec{z}_o||_p \leq ||\vec{z}||_p$, i.e.,
    $||\vec{z}_o||_p$ is the attack whose perturbation $\vec{\delta}$ has minimal norm.
\end{enumerate}
\end{lemma}
\begin{proof}
(1) If $\vec{\delta}$ is undefined then
$\Call{Reachable}{t,p,k,\vec{x},y}=\varnothing$.
This means there is no perturbation $\vec{\delta'}=\vec{z}-\vec{x}$
computed during a call to \Call{Reachable}{}
such that $||\vec{\delta'}|| \leq k$
and is able to push $\vec{z}$ into a leaf with label different from $y$.
That is, $t(\vec{z}) \neq y$.
(2) If $\vec{\delta}$ is defined, then $||\vec{z_o}-\vec{x}|| = ||\vec{\delta}|| \leq k$,
hence $\vec{z_o} \in A_{p,k}(\vec{x})$.
In particular, $\vec{z}_o = \vec{x}+\vec{\delta}$ is the attack of minimum cost
since $\vec{\delta}$ is the perturbation whose norm is minimal, hence it must be $t(\vec{z_o}) \neq y$
and $||\vec{\delta}||_p \leq ||\vec{\delta'}||_p$ for any other $\vec{\delta'}$
such that $\vec{z}=\vec{x}+\vec{\delta'} \in A_{p,k}(\vec{x})$.
But $||\vec{\delta}||_p \leq ||\vec{\delta'}||_p \iff ||\vec{z_o}-\vec{x}||_p \leq ||\vec{z}-\vec{x}||_p \iff ||\vec{z_o}||_p \leq ||\vec{z}||_p$.
\end{proof}

In the following we refer to $\vec{\delta} = \opt(t,p,k,\vec{x},y)$
that is the perturbation whose norm is minimal as just the ``minimal perturbation''.

\begin{lemma}
\label{lem:spread}
Assume $\spread(T) > 2k$ and consider two distinct trees $t,t' \in T$. Let $\vec{x}$ be an instance with label $y$
and let $\vec{\delta} = \opt(t,p,k,\vec{x},y)$ and $\vec{\delta}' = \opt(t',p,k,\vec{x},y)$ be defined, then $\supp(\vec{\delta}) \cap \supp(\vec{\delta}') = \varnothing$.
\end{lemma}
\begin{proof}
By contradiction. Assume that there exists a feature $f$ such that $\delta_f \neq 0$ and $\delta_f' \neq 0$. Since $\vec{\delta},\vec{\delta}'$ are the minimal perturbations with respect to $||\cdot||_p$ by Lemma~\ref{lem:trees-reach}, the feature $f$ must be tested when performing the predictions $t(\vec{x} + \vec{\delta})$ and $t'(\vec{x} + \vec{\delta}')$ respectively, i.e., the prediction $t(\vec{x} + \vec{\delta})$ must perform a test $x_f \leq v$ and the prediction $t'(\vec{x} + \vec{\delta}')$ must perform a test $x_f \leq v'$ for some thresholds $v,v'$.
Since $||v' - x_f||_p = ||x_f - v'||_p$, we get $|| v - v' ||_p \leq ||v - x_f||_p + ||x_f - v'||_p$ by triangle inequality.
Since $||v - x_f||_p \leq k$ and $||v' - x_f||_p \leq k$ because both $\vec{\delta},\vec{\delta'}$ are valid
minimal perturbations, this implies $|| v - v' ||_p \leq 2k$ which contradicts the hypothesis $\spread(T) > 2k$.
\end{proof}

\begin{lemma}
\label{lem:spread2}
Assume $\spread(T) > 2k$ and consider two distinct trees $t,t' \in T$. Let $\vec{x}$ be an instance with label $y$
and let $\vec{z} \in A_{p,k}(\vec{x})$ be such that $t(\vec{z}) \neq y$ and $t'(\vec{z}) \neq y$.
Then there exist $\vec{\delta}$ and $\vec{\delta}'$ such:
\begin{enumerate}
    \item $\vec{z}=\vec{x}+\vec{\delta}+\vec{\delta}'$ with $\supp(\vec{\delta}) \cap \supp(\vec{\delta}') = \varnothing$;
    \item $t(\vec{x}+\vec{\delta}) \neq y$;
    \item $t'(\vec{x}+\vec{\delta}') \neq y$.
\end{enumerate}
\end{lemma}
\begin{proof}
Let $\vec{z}=\vec{x}+\vec{\delta}^*$. We show an algorithm that constructs $\vec{\delta}$ and $\vec{\delta}'$ with the desired properties (1) -- (3). In particular, the algorithm identifies how to split $\vec{\delta}^*$ into $\vec{\delta}$ and $\vec{\delta}'$ by partitioning the set of features. The algorithm takes as input the decision tree $t$, $\vec{x}$ and $\vec{z}$, and operates as follows:
\begin{itemize}
    \item Initialize the set $F$ to $\emptyset$;
    \item Let the root of the tree be $\sigma(f,v,t_l,t_r)$;
    \item If $x_f \leq v$ and $z_f > v$: let $F = F \cup \{f\}$ and recursively visit $t_r$;
    \item If $x_f > v$ and $z_f \leq v$: let $F = F \cup \{f\}$ and recursively visit $t_l$;
    \item If $x_f \leq v$ and $z_f \leq v$: recursively visit $t_l$;
    \item If $x_f > v$ and $z_f > v$: recursively visit $t_r$.
\end{itemize}

The algorithm terminates by returning $F$ when the visit reaches a leaf. We then construct $\vec{\delta}$ by setting the features in $F$ to their corresponding value in $\vec{\delta}^*$ and setting the other features to 0; conversely, we construct $\vec{\delta'}$ by setting all the features not in $F$ to their corresponding value in $\vec{\delta}^*$ and setting the other features to 0. We can show that $\vec{\delta}$ and $\vec{\delta}'$ have the desired properties as follows:
\begin{enumerate}
    \item We have $\supp(\vec{\delta}) \cap \supp(\vec{\delta'}) = \emptyset$ by construction. The fact that $\vec{\delta}^* = \vec{\delta} + \vec{\delta}'$ follows by definition of $\vec{\delta}$ and $\vec{\delta'}$, because all the components of $\vec{\delta}^*$ are distributed between them;
    
    \item We have $t(\vec{x} + \vec{\delta}) \neq y$, because $t(\vec{x} + \vec{\delta})$ follows the same prediction path of $t(\vec{z})$ by construction; 

    \item This property follows from the large-spread assumption. In particular, we define $\vec{z}'$ such that $\forall i \in F: z_i' = x_i$ and $\forall j \not\in F: z_j' = z_j$. We have $t'(\vec{z}') \neq y$, because $t'(\vec{z}')$ follows the same prediction path of $t'(\vec{z})$ by construction. Indeed, features in $F$ cannot affect the prediction path of $t'(\vec{z})$ w.r.t. $t'(\vec{x})$, since such features already affect the prediction path of $t(\vec{z})$ w.r.t. $t(\vec{x})$ and $T$ is large-spread. The property $t'(\vec{x} + \vec{\delta}') \neq y$ follows by observing that $\vec{x} + \vec{\delta}' = \vec{z}'$.
\end{enumerate}
\end{proof}

\begin{lemma}
\label{lem:good-spread}
Assume $\spread(T) > 2k$ and pick any $T' \subseteq T$. Let $\vec{x}$ be an instance with label $y$ and assume that $\vec{\delta}_i = \opt(t_i,p,k,\vec{x},y)$ is defined for every tree $t_i \in T'$.
Then for every $\vec{z} \in A_{p,k}(\vec{x})$ such that $\forall t_i \in T': t_i(\vec{z}) \neq y$ we have $||\vec{x} + \sum_i \vec{\delta}_i||_p \leq ||\vec{z}||_p$.
\end{lemma}
\begin{proof}
Let $\vec{\delta} = \vec{z} - \vec{x}$.
We want to show that
$||\sum_i \vec{\delta}_i||_p \leq ||\vec{\delta}||_p$, which yields the conclusion.
Since $\vec{\delta}$ is a perturbation inducing an attack against all trees
in $T'$ and $T'$ is large-spread, then $\vec{\delta}$ can be written as the sum of $|T'|$ pairwise-orthogonal
perturbations $\{\vec{\delta'_i}\}_i$ by Lemma~\ref{lem:spread2}, 
each inducing an attack against $t_i \in T'$, i.e., $\vec{\delta}=\sum_i \vec{\delta'_i}$.
Hence $||\vec{\delta}||_p = ||\sum_i \vec{\delta'_i}||_p = \bigoplus_i ||\vec{\delta'_i}||_p$.
By Lemma~\ref{lem:trees-reach} (part 2), it must be 
$||\vec{\delta'_i}||_p \geq ||\vec{\delta_i}||_p$ because
$\vec{\delta_i}=\opt(t_i,p,k,\vec{x},y)$.
Since $\{ \vec{\delta_i} \}_i$ are also pairwise-orthogonal by Lemma~\ref{lem:spread}, we conclude $||\vec{\delta}||_p = \bigoplus_i ||\vec{\delta'_i}||_p \geq \bigoplus_i ||\vec{\delta_i}||_p = ||\sum_i \vec{\delta}_i||_p$. 
\end{proof}

\begin{lemma}
\label{lem:soundness}
If $\Call{Stable}{T,p,k,\vec{x},y}$ returns True and $\spread(T) > 2k$, then $T$ is stable on $\vec{x}$ against $A_{p,k}$.
\end{lemma}
\begin{proof}
We perform a case analysis:
\begin{itemize}
    \item \emph{The number of unstable trees is at least $\frac{m-1}{2}+1$}.
    This means that there exist at least $\frac{m-1}{2} + 1$ trees $t_i \in T$ such that $\Call{Reachable}{t_i,p,k,\vec{x},y} \neq \emptyset$, let $\Delta_i$ be the minima of these sets.
    The quantity $\Delta$ computed by the algorithm is the $L_p$-norm of the vector $\vec{\delta} = \sum_{i=1}^{(m-1)/2+1} \opt(t_i,p,k,\vec{x},y)$ by Facts 1-3.
    Note that it must be $\Delta > k$ because the algorithm returns True.
    Assume now by contradiction that there exists $T' \subseteq T$ with $|T'| \geq \frac{m-1}{2}+1$ such that there exist $\vec{z} \in A_{p,k}(\vec{x})$ such that $\forall t' \in T': t'(\vec{z}) \neq y$, i.e., $T$ is not stable on $\vec{x}$. Lemma~\ref{lem:good-spread} ensures that $\Delta \leq ||\vec{z} - \vec{x}||_p$, hence $||\vec{z} - \vec{x}||_p > k$ by transitivity. This is a contradiction, because $\vec{z} \in A_{p,k}(\vec{x})$, i.e., we conclude that $T$ is stable on $\vec{x}$.
    
    \item \emph{The number of unstable trees is less than $\frac{m-1}{2}+1$}.
    This means that there are less than $\frac{m-1}{2}+1$ trees $t_i \in T$
    such that $\Call{Reachable}{t_i,p,k,\vec{x},y} \neq \emptyset$.
    By Lemma~\ref{lem:trees-reach} (part 1) there exist at least $\frac{m-1}{2}+1$ trees $t_i \in T$ such that $\forall \vec{z} \in A_{p,k}(\vec{x}): t_i(\vec{z}) = y$. This implies $\forall \vec{z} \in A_{p,k}(\vec{x}): T(\vec{z}) = y$, i.e., $T$ is stable on $\vec{x}$.
\end{itemize}
\end{proof}

\begin{lemma}
\label{lem:composition}
Assume $\spread(T) > 2k$ and pick any $T' \subseteq T$. Let $\vec{x}$ be an instance with label $y$ and assume that $\vec{\delta}_i = \opt(t_i,p,k,\vec{x},y)$ is defined for every tree $t_i \in T'$. If $\vec{z} = \vec{x} + \sum_{i} \vec{\delta}_i \in A_{p,k}(\vec{x})$, then $\forall t_i \in T': t_i(\vec{z}) \neq y$.
\end{lemma}
\begin{proof}
Let $T' = \{t_1,...,t_q\}$, the proof is by induction on $q$. The base case is $q = 1$. Since $\vec{\delta}_1 = \opt(t_1,p,k,\vec{x},y)$ is defined, Lemma~\ref{lem:trees-reach} (part 1) guarantees that there exists $\vec{z}_1 = \vec{x} + \vec{\delta}_1 \in A_{p,k}(\vec{x})$ such that $t_1(\vec{z}_1) \neq y$. The inductive case is when $q > 1$. Let $\vec{\delta}_i = \opt(t_i,p,k,\vec{x},y)$ be defined for all $i$. Note that $\sum_{i = 1}^q \vec{\delta}_i =  \sum_{i = 1}^{q-1} \vec{\delta}_i + \vec{\delta}_q$, at first we show that $||\sum_{i = 1}^{q-1} \vec{\delta}_i||_p \leq k$. Since $supp(\sum_{i=1}^{q-1}\vec{\delta}_i) = \bigcup_{i=1}^{q-1}supp(\vec{\delta}_i)$ and $\forall 1 \leq i \leq q-1: supp(\vec{\delta}_i) \cap supp(\vec{\delta}_q) = \emptyset$ by Lemma~\ref{lem:spread}, we have that $supp(\sum_{i=1}^{q-1}\vec{\delta}_i) \cap supp(\vec{\delta}_q) = \emptyset$. By observing that $||\sum_{i = 1}^q \vec{\delta}_i||_p \leq k$ by definition and $supp(\sum_{i=1}^{q-1}\vec{\delta}_i)$ and $supp(\vec{\delta}_q)$ are disjoint, we have that $||\sum_{i = 1}^{q-1} \vec{\delta}_i||_p \leq k$. This implies $\vec{x} + \sum_{i = 1}^{q-1} \vec{\delta}_i \in A_{p,k}(\vec{x})$, hence $\forall 1 \leq i \leq q - 1: t_i(\vec{x} + \sum_{i = 1}^{q-1}\vec{\delta}_i) \neq y$ by induction hypothesis. Moreover, we know that $t_q(\vec{x} + \vec{\delta}_q) \neq y$ by Lemma~\ref{lem:trees-reach} (part 1) and $\vec{z} = \vec{x} + \sum_{i = 1}^q \vec{\delta}_i \in A_{p,k}(\vec{x})$ since $||\sum_{i = 1}^q \vec{\delta}_i||_p \leq k$ by definition. We now show that $\forall 1 \leq i \leq q: t_i(\vec{z}) \neq y$. Since $supp(\sum_{i=1}^{q-1} \vec{\delta}_i)$ and $supp(\vec{\delta}_q)$ are disjoint, for every feature $f$ we have $(\sum_{i=1}^{q} \vec{\delta}_i)_f = (\sum_{i=1}^{q-1} \vec{\delta}_i)_f$ or $(\sum_{i=1}^{q} \vec{\delta}_i)_f = (\vec{\delta}_q)_f$. In addition, we observe that:
\begin{itemize}
    \item $t_q(\vec{z})$ follows the same prediction path of $t_q(\vec{x} + \vec{\delta}_q)$. Indeed features in $supp(\sum_{i=1}^{q-1} \vec{\delta}_i)$ cannot affect the prediction path of $t_q(\vec{z})$ w.r.t $t_q(\vec{x})$ because they already affect the prediction path of $\vec{z}$ in the trees in $\{t_1, \dots, t_{q-1}\}$ w.r.t the path of $\vec{x}$ in the trees in $\{t_1, \dots, t_{q-1}\}$ and $T$ is large-spread.
    \item for every $1 \leq i \leq q-1$, $t_i(\vec{z})$ follow the same prediction path of $t_i(\vec{x} + \sum_{i=1}^{q-1} \vec{\delta_i})$. Indeed features in $supp(\vec{\delta}_q)$ cannot affect the prediction path of $\vec{z}$ in the trees in $\{t_1, \dots, t_{q-1}\}$ w.r.t the path of $\vec{x}$ in the trees in $\{t_1, \dots, t_{q-1}\}$ because they already affect the prediction path of $t_q(\vec{z})$ w.r.t $t_q(\vec{x})$ and $T$ is large-spread.
\end{itemize}
We thus conclude that $\forall 1 \leq i \leq q: t_i(\vec{z}) \neq y$. 
\end{proof}

\begin{lemma}
\label{lem:completeness}
If $\:T$ is stable on $\vec{x}$ against $A_{p,k}$ and $\spread(T) > 2k$, then $\Call{Stable}{T,p,k,\vec{x},y}$ returns True.
\end{lemma}
\begin{proof}
By contradiction. Assume that $\Call{Stable}{T,p,k,\vec{x},y}$ returns False. This means that there exist at least $\frac{m-1}{2}+1$ trees $t_i \in T$ such that $\Call{Reachable}{t_i,p,k,\vec{x},y} \neq \emptyset$, let $\Delta_i$ be the minima of these sets. The quantity $\Delta$ computed by the algorithm is the $L_p$-norm of the vector $\vec{\delta} = \sum_{i=1}^{(m-1)/2+1} \opt(t_i,p,k,\vec{x},y)$ by Facts 1-3. Note that it must be $\Delta \leq k$ because the algorithm returns False. Let $T' \subseteq T$ be the set of trees used in the computation of $\Delta$ and let $\vec{z} = \vec{x} + \vec{\delta}$. We have that $\vec{z} \in A_{p,k}(\vec{x})$ and $\forall t_i \in T': t_i(\vec{z}) \neq y$ by Lemma~\ref{lem:composition}, hence $T$ is not stable on $\vec{x}$ against $A_{p,k}$. This is a contradiction, since $T$ is stable on $\vec{x}$ against $A_{p,k}$, i.e., we conclude that $\Call{Stable}{T,p,k,\vec{x},y}$ returns True.
\end{proof}

\ensembles*
\begin{proof}
Since $\spread(T) > 2k$, Lemma~\ref{lem:soundness} and Lemma~\ref{lem:completeness} ensure that $\Call{Stable}{T,p,k,\vec{x},y}$ returns True if and only if $T$ is stable on $\vec{x}$ against the attacker $A_{p,k}(\vec{x})$. The conclusion follows by observing that $T$ is robust on $\vec{x}$ if and only if $T(\vec{x}) = y$ and $T$ is stable on $\vec{x}$.
\end{proof}

\section{Proof of Theorem~\ref{thm:np-hard}}
\label{sec:proofs2}
We here prove the NP-hardness of the large-spread subset problem.

\nphard*
\begin{proof}
We show how to reduce an instance of the max-clique problem to the large-spread subset problem. Let $G = (V,E)$ be an undirected graph, we construct an ensemble $T$ such that $G$ contains a clique of size $s$ if and only if there exists $T' \subseteq T$ of size $s$ such that $T'$ is large-spread for the attacker $A_{0,0}$ (who does not perturb any feature). In particular, the construction operates as follows:
\begin{enumerate}
    \item We construct the complementary graph $\overline{G} = (V,\overline{E})$, where $\{u,v\} \in \overline{E}$ if and only if $\{u,v\} \not\in E$.
    \item We introduce a feature $\pi_E(e)$ for each $e \in \overline{E}$, i.e., the dimensionality of the feature space is $|\overline{E}|$. We assume features are totally ordered, using any ordering convention.
    \item We construct a decision tree $\pi_V(v)$ for each $v \in V$ as follows:
    \begin{itemize}
        \item If $\deg(v) = 0$, we let $\pi_V(v)$ be a leaf with label +1.
        \item If $\deg(v) > 0$, we let $\pi_V(v)$ be a decision tree of depth $\deg(v)$ built on top of the feature set $\{\pi_E(e) ~|~ \exists u \in V: e = \{u,v\} \in \overline{E}\}$. Each level of the tree tests a single feature from this set with threshold arbitrarily set to 1, following the total order assumed on features. Leaves are arbitrarily set so that the left child has label -1 and the right child has label +1.
    \end{itemize}
\end{enumerate}

By construction, decision trees $\pi_V(u),\pi_V(v) \in T$ share a feature if and only if $\{u,v\} \not\in E$, i.e., $\pi_V(u),\pi_V(v) \in T$ do not share any feature if and only if $\{u,v\} \in E$. Given that all thresholds are set to the same value 1 and the attacker does not perturb any feature, any $T' \subseteq T$ is large-spread if and only if the trees in $T'$ do not share any feature. This implies that $G$ has a clique of size $s$ if and only if there exists $T' \subseteq T$ of size $s$ such that $T'$ is large-spread.
\end{proof}
\section{Parameter Tuning for LSE Training}
\label{sec:appendix-tuning}
\revise{Our training algorithm for large-spread ensembles has four hyper-parameters: the maximum number of iterations $MAX\_ITER$ to fix violations to the large-spread condition, the multiplicative factor $MULT$ determining the size of the initially trained forest, the interval $INTV$ of the perturbation applied to fix the forest and the size of the feature partition $l$. Each hyper-parameter can affect the performance of the trained large-spread ensemble, as well as the successful termination of the training algorithm.} 

\revise{As it is customary for tree-based models, we deal with hyper-parameter tuning by means of \emph{grid search}, i.e., we try out all the possible combinations of specific hyper-parameter values to identify the one performing best on a validation set including $20\%$ of the training data, extracted via stratified random sampling. Specifically, we look for the combination of hyper-parameters optimizing the average between accuracy and robustness on the validation set, and we perform a grid search by considering the following possible values for the hyper-parameters:
\begin{itemize}
\item $MAX\_ITER \in \{100, 500\}$
\item $MULT \in \{2,4,6\}$
\item $INTV \in \{[0.5k, k], [k,1.5k]\}$
\item $l \in \{1,2,3,4,5,6\}$
\end{itemize}}

\begin{table*}[t]
\caption{Grid search results for large-spread ensembles trained using the LSE tool. For each dataset, model size (number of trees \& maximum depth) and perturbation $k$, the table reports the value of the hyper-parameters leading to the highest accuracy on the validation set. The large-spread ensembles are trained with norm $p=\infty$.}
\label{tab:parameter-tuning}
    \centering
    \begin{tabular}{c|c|c|c|c|c|c|c|c|c}
    \toprule
    \textbf{Dataset} & \textbf{Trees} & \textbf{Depth} & $k$ & $MAX\_ITER$ & $MULT$ & $INTV$ & $l$ & \textbf{Accuracy} & \textbf{Robustness} \\
    \midrule
    \multirow{6}{*}{\revise{Fashion-MNIST}} & \multirow{3}{*}{\revise{25}} & \multirow{3}{*}{\revise{4}} & \revise{0.0050} & \revise{100} & \revise{4} & \revise{$[0.5k,k]$} & \revise{1} & \revise{0.92} & \revise{0.90} \\
    & & & \revise{0.0100} & \revise{100} & \revise{4} & \revise{$[0.5k,k]$} & \revise{1} & \revise{0.92} & \revise{0.87} \\
    & & & \revise{0.0150} & \revise{100} & \revise{6} & \revise{$[k,1.5k]$} & \revise{1} & \revise{0.91} & \revise{0.88} \\
    \cline{2-10}
    & \multirow{3}{*}{\revise{101}} & \multirow{3}{*}{\revise{6}} & \revise{0.0050} & \revise{500} & \revise{2} & \revise{$[0.5k,k]$} & \revise{1} & \revise{0.96} & \revise{0.93}\\
    & & & \revise{0.0100} & \revise{500} & \revise{6} & \revise{$[0.5k,k]$} & \revise{1} & \revise{0.94} & \revise{0.91} \\
    & & & \revise{0.0150} & \revise{100} & \revise{4} & \revise{$[0.5k,k]$} & \revise{5} & \revise{0.92} & \revise{0.89} \\
    \midrule
    \multirow{6}{*}{\revise{MNIST}} & \multirow{3}{*}{\revise{25}} & \multirow{3}{*}{\revise{4}} & \revise{0.0050} & \revise{100} & \revise{6} & \revise{$[0.5k,k]$} & \revise{1} & \revise{0.97} & \revise{0.96} \\
    & & & \revise{0.0100} & \revise{100} & \revise{2} & \revise{$[0.5k,k]$} & \revise{1} & \revise{0.97} & \revise{0.90} \\
    & & & \revise{0.0150} & \revise{100} & \revise{2} & \revise{$[0.5k,k]$} & \revise{1} & \revise{0.97} & \revise{0.83} \\
    \cline{2-10}
    & \multirow{3}{*}{\revise{101}} & \multirow{3}{*}{\revise{6}} & \revise{0.0050} & \revise{100} & \revise{6} & \revise{$[0.5k,k]$} & \revise{1} & \revise{0.99} & \revise{0.97} \\
    & & & \revise{0.0100} & \revise{500} & \revise{6} & \revise{$[0.5k,k]$} & \revise{1} & \revise{0.99} & \revise{0.97} \\
    & & & \revise{0.0150} & \revise{500} & \revise{2} & \revise{$[0.5k,k]$} & \revise{4} & \revise{0.99} & \revise{0.94} \\
    \midrule
    \multirow{6}{*}{\revise{REWEMA}} & \multirow{3}{*}{\revise{25}} & \multirow{3}{*}{\revise{4}} & \revise{0.0050} & \revise{100} & \revise{2} & \revise{$[0.5k,k]$} & \revise{1} & \revise{0.88} & \revise{0.87}\\
    & & & \revise{0.0100} & \revise{100} & \revise{6} & \revise{$[k,1.5k]$} & \revise{1} & \revise{0.88} & \revise{0.87}\\
    & & & \revise{0.0150} & \revise{100} & \revise{6} & \revise{$[k,1.5k]$} & \revise{1} & \revise{0.88} & \revise{0.85}\\
    \cline{2-10}
    & \multirow{3}{*}{\revise{101}} & \multirow{3}{*}{\revise{6}} & \revise{0.0050} & \revise{500} & \revise{2} & \revise{$[0.5k,k]$} & \revise{1} & \revise{0.89} & \revise{0.89} \\
    & & & \revise{0.0100} & \revise{500} & \revise{4} & \revise{$[0.5k,k]$} & \revise{1} & \revise{0.89} & \revise{0.88} \\
    & & & \revise{0.0150} & \revise{500} & \revise{4} & \revise{$[k,1.5k]$} & \revise{2} & \revise{0.88} & \revise{0.88} \\
    \midrule
    \multirow{6}{*}{\revise{Webspam}} & \multirow{3}{*}{\revise{25}} & \multirow{3}{*}{\revise{4}} & \revise{0.0002} & \revise{100} & \revise{2} & \revise{$[k,1.5k]$} & \revise{1} & \revise{0.90} & \revise{0.87}\\
    & & & \revise{0.0004} & \revise{100} & \revise{2} & \revise{$[0.5k,k]$} & \revise{1} & \revise{0.89} & \revise{0.86}\\
    & & & \revise{0.0006} & \revise{100} & \revise{2} & \revise{$[k,1.5k]$} & \revise{1} & \revise{0.89} & \revise{0.85} \\
    \cline{2-10}
    & \multirow{3}{*}{\revise{101}} & \multirow{3}{*}{\revise{6}} & \revise{0.0002} & \revise{500} & \revise{4} & \revise{$[0.5k,k]$} & \revise{5} & \revise{0.91} & \revise{0.90}\\
    & & & \revise{0.0004} & \revise{500} & \revise{4} & \revise{$[0.5k,k]$} & \revise{6} & \revise{0.89} & \revise{0.86}\\
    & & & \revise{0.0006} & \revise{500} & \revise{4} & \revise{$[0.5k,k]$} & \revise{6} & \revise{0.85} & \revise{0.82}\\
    \bottomrule
    \end{tabular}
\end{table*}

\revise{Table~\ref{tab:parameter-tuning} reports for each dataset, model size (trees and depth) and perturbation $k$ the value of the hyper-parameters leading to the best performance on the validation set. By looking at the results, we can catch some insights about the influence of each hyper-parameter on training the large-spread ensembles.}

\revise{We first examine the values of $MAX\_ITER$, the number of iterations needed to train the best-performing large-spread ensembles. We observe that the chosen value of $MAX\_ITER$ depends on the size of the model: typically, only $100$ iterations are needed to train the best-performing large-spread ensembles of 25 trees, while $500$ iterations are needed for training the best-performing large-spread ensembles of 101 trees. The intuitive reason is that more iterations are needed for successfully training large-spread ensembles with many trees, since more thresholds need to be adjusted to fulfill the large-spread condition. Large-spread ensembles with fewer trees can be trained even with $100$ iterations, instead, and a lower number of iterations is often beneficial there, since less noise needs to be applied to adjust the original thresholds.}

\revise{As to the size of the feature partition $l$, the results show that $l = 1$ is the value leading to more than half of the best-performing large-spread ensembles. In particular, $l=1$ is used for training almost all the best large-spread ensembles on MNIST, Fashion-MNIST and REWEMA. This suggests that avoiding partitioning the features is the best choice for training the best large-spread ensembles on most datasets: an ensemble trained on all the available features may exhibit better accuracy and robustness than the ones of an ensemble built of sub-forests trained on subsets of features, since the sub-forests have only a partial view of the set of available features and some patterns might not be learned. Nevertheless, partitioning the features can be useful for training the best-performing large-spread ensembles in some cases. For example, when training the best-performing large-spread ensembles with $101$ trees and maximum depth six, the best choice is $l = 4$ on the MNIST dataset when considering the perturbation $k=0.0150$, while $l \in \{5,6\}$ for all the models trained on the Webspam dataset. The result highlights the potential benefits of hierarchical training: training smaller large-spread ensembles on subsets of features and then merging them can enable the training of bigger large-spread ensembles.}

\revise{Finally, we observe that the values of $MULT$ and $INTV$ used for training the best-performing large-spread ensembles are strongly dependent on the specific experiment and do not show particularly insightful patterns. It is common to identify different optimal values of $MULT$ and $INTV$ even for the same dataset and model size. Grid search is a standard practice to deal with the unpredictable nature of these hyper-parameters.}

\fi

\end{document}